\documentclass[journal]{IEEEtran} 

\usepackage{amsmath,amsfonts,bm}
\usepackage{amsthm,amsmath,amssymb}
\usepackage{eucal}
\usepackage{xifthen}
\usepackage{mathtools}
\usepackage{enumerate}
\usepackage{microtype}
\usepackage{xspace}
\usepackage{bm}
\usepackage{cite}
\usepackage{fancyhdr}
\usepackage{lastpage}
\usepackage[small]{caption}
\usepackage{xcolor}
\usepackage{algorithm}
\usepackage{algpseudocode}
\allowdisplaybreaks

\def\eqref#1{Eq-(\ref{#1})}

\newcommand{\nn}{\nonumber}

\newtheorem{remark}{Remark}
\newtheorem{theorem}{Theorem}
\newtheorem{lemma}{Lemma}

\def\1{\bm{1}}

\allowdisplaybreaks

\DeclareMathAlphabet{\mathcal}{OMS}{cmsy}{m}{n}

\def\vzero{{\bm{0}}}
\def\vone{{\bm{1}}}

\def\va{{\bm{a}}}

\def\vu{{\bm{u}}}
\def\vv{{\bm{v}}}
\def\vw{{\bm{w}}}
\def\vx{{\bm{x}}}
\def\vy{{\bm{y}}}

\def\mA{{\bm{A}}}
\def\mB{{\bm{B}}}
\def\mC{{\bm{C}}}
\def\mD{{\bm{D}}}

\def\mG{{\bm{G}}}
\def\mH{{\bm{H}}}
\def\mI{{\bm{I}}}
\def\mJ{{\bm{J}}}

\def\mL{{\bm{L}}}

\def\mN{{\bm{N}}}

\def\mR{{\bm{R}}}

\def\mU{{\bm{U}}}
\def\mV{{\bm{V}}}
\def\mW{{\bm{W}}}
\def\mX{{\bm{X}}}
\def\mY{{\bm{Y}}}
\def\mZ{{\bm{Z}}}

\def\mLambda{{\bm{\Lambda}}}

\DeclareMathAlphabet{\mathsfit}{\encodingdefault}{\sfdefault}{m}{sl}
\SetMathAlphabet{\mathsfit}{bold}{\encodingdefault}{\sfdefault}{bx}{n}

\def\gA{{\mathcal{A}}}

\def\gG{{\mathcal{G}}}
\def\gH{{\mathcal{H}}}

\def\gN{{\mathcal{N}}}
\def\gO{{\mathcal{O}}}

\def\gS{{\mathcal{S}}}

\def\gX{{\mathcal{X}}}

\def\gZ{{\mathcal{Z}}}

\renewcommand{\vec}{\mathrm{vec}}

\newcommand{\E}{\mathbb{E}}

\newcommand{\R}{\mathbb{R}}

\newcommand{\KL}{D_{\mathrm{KL}}}

\newcommand{\Norm}[1]{\|#1\|}
\newcommand{\Abs}[1]{|#1|}

\DeclareMathOperator{\Tr}{Tr}

\usepackage{hyperref}
\usepackage{url}
\usepackage{amssymb}
\usepackage{amsmath}

\title{Topology-Aware PAC-Bayesian Generalization Analysis for Graph Neural Networks}

\author{Xinping Yi, \emph{Member, IEEE} 
\thanks{X. Yi is with the School of Information Science and Engineering, Southeast University, Nanjing, China. Email: \texttt{xyi@seu.edu.cn}.}
}

\begin{document}

\maketitle

\begin{abstract}
Graph neural networks have demonstrated excellent applicability to a wide range of domains, including social networks, biological systems, recommendation systems, and wireless communications. Yet a principled theoretical understanding of their generalization behavior remains limited, particularly for graph classification tasks where complex interactions between model parameters and graph structure play a crucial role. Among existing theoretical tools, PAC-Bayesian norm-based generalization bounds provide a flexible and data-dependent framework; however, current results for GNNs often restrict the exploitation of graph structures. In this work, we propose a topology-aware PAC-Bayesian norm-based generalization framework for graph convolutional networks (GCNs) that extends a previously developed framework to graph-structured models. Our approach reformulates the derivation of generalization bounds as a stochastic optimization problem and introduces sensitivity matrices that measure the response of classification outputs with respect to structured weight perturbations. By imposing different structures on sensitivity matrices from both spatial and spectral perspectives, we derive a family of generalization error bounds with graph structures explicitly embedded. Such bounds could recover existing results as special cases, while yielding bounds that are tighter than state-of-the-art PAC-Bayesian bounds for GNNs. Notably, the proposed framework explicitly integrates graph structural properties into the generalization analysis, enabling a unified inspection of GNN generalization behavior from both spatial aggregation and spectral filtering viewpoints.

\end{abstract}


\section{Introduction}
Despite the wide adoption and strong empirical performance of graph neural networks in learning from relational and structured data across domains such as social networks, biology, and wireless networks, fundamental questions as to how and why these models generalize beyond the training graphs are still largely unexplored.
From a statistical learning perspective, the nature of graph data, together with weight sharing and topology-dependent message passing, poses fundamental challenges for classical generalization analysis. To address these challenges, a variety of theoretical frameworks have been explored for GNNs, including VC dimension–based capacity control \cite{vapnik2015uniform,scarselli2018vapnik}, data-dependent measures such as Rademacher complexity \cite{bartlett2002rademacher,garg2020generalization}, algorithmic stability \cite{bousquet2002stability,verma2019stability}, and PAC-Bayesian approaches \cite{mcallester2003simplified,liao2020pac} tailored to dependent samples and graph structures. 

Among existing approaches, PAC-Bayesian theory has emerged as one of the most promising tools for analyzing deep learning models, owing to its flexibility, data-dependent nature, and ability to produce non-vacuous bounds in over-parameterized regimes. A growing body of work has demonstrated the effectiveness of PAC-Bayesian analysis for neural networks (e.g.,\cite{neyshabur2015norm,dziugaite2017computing,neyshabur2018pac,zhou2018non,lotfi2022pac}) and, more recently, for graph-structured models (e.g., \cite{liao2020pac,ju2023generalization,sun2024pac,brilliantov2024compositional,wang2025generalization}), motivating the development of a topology-aware PAC-Bayesian generalization framework that explicitly accounts for the architectural and spectral properties unique to graph neural networks.

\subsection{Related Works}

Theoretical analysis of generalization for graph neural networks has become an active research area in statistical learning, driven by the wide adoption of GNNs across node-level and graph-level prediction tasks. Early attempts to characterize GNN generalization relied on classical capacity measures such as the VC dimension (e.g., \cite{scarselli2018vapnik} for a specific class of GNNs and \cite{esser2021learning} for node classification in the transductive setting), which directly bounds the generalization gap but often scales unfavorably with model and graph size. More refined analyses have used Rademacher complexity (e.g., \cite{garg2020generalization,karczewski2024generalization} for message passing GNNs in the supervised setting, \cite{esser2021learning} for node classification in the transductive setting, \cite{lv2021generalization} for one-layer GCNs in the semi-supervised node classification) and algorithmic stability (e.g., \cite{verma2019stability} for 1-layer GNNs in the transductive setting, and the multi-layer extension in \cite{zhou2021generalization}) to obtain data-dependent bounds that capture graph sampling and aggregation dynamics. There also exist some other approaches such as \cite{du2019graph} considered an infinitely wide multi-layer GNN with neural tangent kernel, and a number of works leverages the covering number \cite{levie2023graphon,maskey2022generalization,maskey2025generalization,rauchwerger2024generalization,vasileiou2024covered}. Please refer to the survey of generalization for GNNs \cite{vasileiou2025survey} and references therein.

Of particular relevance is the norm-based PAC-Bayesian framework, which has been successfully extended to GNNs in graph classification tasks for high-probability generalization guarantees that incorporate both model complexity and uncertainty in learned parameters. 
Recent works demonstrated that graph-specific factors such as maximum node degree, feature distribution, and spectral properties of diffusion operators can significantly influence generalization behavior.
A notable example is the work by Liao et al. \cite{liao2020pac}, which derives PAC-Bayesian generalization bounds for major classes of GNNs, including graph convolutional networks (GCNs) \cite{kipf2017semi} and message passing GNNs (MPGNNs) \cite{gilmer2017neural}. Their results reveal that both the maximum node degree and the spectral/Frobenius norms of weight matrices govern the generalization bound. 
Subsequent PAC-Bayesian analyses have further refined this perspective, for example, \cite{sun2024pac} addressed robust generalization under adversarial perturbations in graph data, while still leveraging spectral properties in the bound formulation.
Ju et al. \cite{ju2023generalization} developed improved PAC-Bayesian generalization bounds that scale with the largest singular value of the graph feature diffusion matrix rather than solely with node degrees. 
Importantly, the use of spectral norms of graph operators and weight matrices aligns well with empirical observations about generalization in graph-level tasks and offers sharper insight into how graph structure impacts predictive performance.
Most recently, such a norm-based PAC-Bayesian framework has been extended to derive generalization bounds for topological \cite{brilliantov2024compositional} and hypergraph \cite{wang2025generalization} neural networks.
Notably, these results rely critically on extending the perturbation bound in \cite{neyshabur2018pac} to graph-structured models with matrix spectral norm concentration, which limits potential thorough inspection into the impact of network architectures.

\subsection{Motivation and Contributions}
\textbf{Motivation.}
Despite the growing body of PAC-Bayesian generalization results for graph neural networks, existing analyses remain subject to several limitations that motivate the current work. First, most prior works assume isotropic Gaussian posteriors with homogeneous uncertainty across all parameters, which is particularly misaligned with GNNs, where weight sharing, depth-dependent aggregation, and graph-induced correlations lead to highly anisotropic and layer-specific sensitivity patterns in practice. Second, existing PAC-Bayesian bounds for GNNs often rely heavily on controlling the spectral norms of weight perturbations and graph propagation operators, where the corresponding spectral norm concentration inequalities are known to be brittle for deep, wide, or structured architectures, therefore limiting the tightness of the resulting guarantees. Third, existing approaches typically incorporate graph structure only through coarse quantities such as maximum degree or global spectral radius, and rarely account for the heterogeneous influence of different graph components or the diverse sensitivity of the loss landscape with respect to individual weights and layers. These limitations collectively suggest that a topology-aware PAC-Bayesian norm-based generalization framework for GNNs—one that allows anisotropic posteriors, refined perturbation control beyond global spectral norms, and explicit integration of graph structure and parameter-wise sensitivity—is essential for a more theoretically grounded understanding of generalization in graph neural networks.

\textbf{Contributions.} 
Building on the above motivation, this work makes the following contributions toward a topology-aware theory of PAC-Bayesian norm-based generalization for graph neural networks, with a particular focus on graph convolutional networks (GCNs) for graph classification:
\begin{itemize}
\item \textbf{Topology-aware generalization bounds within a unified framework.} 
We demonstrate that the previously proposed unified PAC-Bayesian norm-based generalization framework for deep neural networks in \cite{yi2026unified} can be naturally extended to graph-structured models, explicitly accounting for the message-passing and aggregation mechanisms inherent to GCNs.
By introducing sensitivity matrices that capture the graph classification differences to weight perturbations in GCNs, we derive a family of PAC-Bayesian generalization bounds under different structural considerations. In particular, imposing diagonal, low-rank, or graph-aligned structures on these matrices—motivated from both spatial (node- and neighborhood-level aggregation) and spectral (graph Fourier and operator-norm) perspectives—yields a spectrum of generalization error bounds. These bounds recover existing PAC-Bayesian results for GNNs as special cases and are shown to be tighter than state-of-the-art bounds that rely on isotropic perturbations or global spectral-norm control.
\item \textbf{Explicit integration of graph structure into generalization analysis.} Our framework explicitly integrates graph structural information into the generalization bounds, enabling a transparent inspection of how topology, diffusion operators, and spectral properties of the graph influence generalization performance. By either capturing spatial aggregation effects or spectral filtering behavior, the resulting bounds provide new insights into the role of graph structure in controlling model complexity and robustness, offering a distinct lens through which the generalization behavior of GNNs can be analyzed from both spatial and spectral viewpoints.
\end{itemize}

The paper is organized as follows: Section I presents some preliminaries for graph neural networks and the PAC-Bayesian margin-based bounds, followed by the extension of a unified framework to the graph settings. Section IV details the specific designs of different sensitivity matrices, ended up with various PAC-Bayesian generalization bounds under the unified framework. Section V concludes the paper, and the Appendix gives some detailed proofs of key results used in Section III.

\textbf{Notation.} We interchangeably use random variables and their realizations unless they are unclear from the context.
We use $a$, $\va$, $\mA$, and $\gA$ to represent the scalar, vector, matrix and set, respectively. Accordingly, $\va_i$ is the $i$-th element of $\va$, $\mA_{i,:}$ is the $i$-th row of $\mA$, and $\mA_{ij}$ is the element at the $i$-th row and the $j$-th column in $\mA$. The vectorization of a matrix, i.e., $\vec(\mA)$, is the concatenation of the columns of $\mA$ as a long vector. We use $\vone$, $\vzero$, $\mathbf{0}$, and $\mI_n$ to denote the all-one, all-zero vectors, all-zero matrix, and the $n \times n$ identity matrix, respectively, and the size of the vectors is determined by the context. 
The operators $\Tr(\cdot)$ and $\det(\cdot)$ are the trace and determinant of a square matrix, respectively.
Given a vector $\va$, the $\ell_2$ and $\ell_{\infty}$ vector norm are given by $\Norm{\va}_2=\sqrt{\sum_i \va_i^2}$ and $\Norm{\va}_{\infty}=\max_i \Abs{\va_i}$, with $\Norm{\va}_{\infty} \le \Norm{\va}_2$. Given a matrix $\mA$, we use $\Norm{\mA}_2$, $\Norm{\mA}_F$, $\Norm{\mA}_{2,\infty}$ to represent the spectral norm, the Frobenius norm, and the $(2,\infty)$ norm, such that $\Norm{\mA}_2=\max_{\vx \neq \vzero} \frac{\Norm{\mA \vx}_2}{\Norm{\vx}_2}$, $\Norm{\mA}_F^2=\sum_i \sum_j \mA_{ij}^2 = \Tr(\mA^T\mA)$, and $\Norm{\mA}_{2,\infty}=\max_i \Norm{\mA_{i,:}}_2^2$, which is the maximum value of the vector norms of the rows of $\mA$. 
For any vector $\vx$ and any matrices $\mA$ and $\mB$, we have $\Norm{\mA\vx}_2 \le \Norm{\mA}_2 \Norm{\vx}_2$ and $\Norm{\mA \mB}_2 \le \Norm{\mA}_2 \Norm{\mB}_2$.
Given two matrices $\mA \in \R^{m \times n}$ and $\mB \in \R^{p \times q}$, the Kronecker product $\mA \otimes \mB$ is a $mp \times nq$ matrix such that the $(i,j)$-th block is $\mA_{ij}\mB$ with size $p \times q$.
For two matrices $\mA$ and $\mB$, there are some properties of the Kronecker product, e.g., $\mA \otimes (a\mB) = a (\mA \otimes \mB)$, $(\mA \otimes \mB)^T=\mA^T \otimes \mB^T$, $(\mA \otimes \mB)(\mC \otimes \mD)=(\mA\mC) \otimes (\mB \mD)$, $\det (\mA_{n \times n} \otimes \mB_{m \times m})=(\det \mA)^m (\det \mB)^n$, $\Tr(\mA \otimes \mB)=\Tr(\mA)\Tr(\mB)$, $\Norm{\mA \otimes \mB}_2 = \Norm{\mA}_2 \Norm{\mB}_2$, $\Norm{\mA \otimes \mB}_F = \Norm{\mA}_F \Norm{\mB}_F$, and $\mathrm{vec}(\mA \mX \mB) = (\mB^T \otimes \mA) \mathrm{vec}(\mX)$, according to Lemma \ref{lemma:kron-norm-eq}.
The matrix $\mA$ is said to be semi-definite, denoted by $\mA \succeq \mathbf{0}$, such that for any $\vx$, $\vx^T \mA \vx \ge 0$ holds. $\mA \succeq \mB$, equivalently $\mB \preceq \mA$, if and only if $\mA-\mB \succeq \mathbf{0}$.

\section{Preliminaries}
\label{gen_inst}

Following the problem setting of \cite{liao2020pac}, we consider the $K$-class graph classification task for simple and undirected graphs $G=(V,E)$ with $n$ nodes.
The input space $\gZ$ consists of graph samples $z=(\mA_G,\mX,y) \in \gZ$ where $\mA_G \in \{0,1\}^{n \times n}$ is the adjacency matrix such that $[\mA_G]_{ij}=1$ if $(i,j) \in E$ is an edge, $\mX \in   \mathbb{R}^{n \times h_0}$ is the node features over $n$ nodes subject to $\Norm{\mX}_{2,\infty} = \max_{i} \Norm{\mX_{i,:}}_2\le B$, and $y \in \{1,\dots,K\}$ is the output.
We denote $y = \arg \max_y f_{\vw}(\mX)[y]$, where $f_{\vw} \in \gH: \mathcal{X} \times \gG \to \mathbb{R}^K$ is a function specified by a parameterized learning model (e.g., graph neural networks) with the collected model parameters $\vw$.\footnote{For simplicity, the bias term is absorbed into the weight matrices by appending constant 1 to the node feature, which is a common trick in deep learning community.} 
The space $\gZ$, $\gX$, $\gG$ and $\gH$ are defined as input sample space, node feature space, graph space, and hypothesis class space, respectively.
A dataset $\mathcal{S}=\{z_1, \dots, z_m\}$ with $m$ training samples drawn identically and independently from an unknown distribution $\mathcal{D}$ is given to learn the model parameters.

For the graph classification task, a graph convolutional network (GCN)-type model is considered such that the node embeddings at the $l$-th graph convolutional layer ($l < d$) and the last readout layer are given by
\begin{align}
   \text{Conv: } \mH_{l} &= \phi({\mL} \mH_{l-1} \mW_l)\\
   \text{Readout: } f_{\vw}^T(\mX) &= \frac{1}{n} \vone^T \mH_{d-1} \mW_d
\end{align}
where $\phi$ is a 1-Lipschitz element-wise activation function, e.g., ReLU and tanh, $\mL$ is the graph diffusion/propagation matrix, $\mH_l \in \R^{n \times h_l}$ is node embedding at the $l$-th GCN layer with $\mH_0=\mX$, and $\mW_l \in \R^{h_{l-1} \times h_{l}}$ is the weight matrix for linear transformation at each node at the $l$-th layer, whose parameters are shared across nodes. When $l=d$, $h_d=K$. By setting $\mL$ to be different matrices, we have various GCN-like models.

\textbf{Vanilla GCN} \cite{kipf2017semi}.
When we consider the normalized adjacency matrix as the graph propagation matrix, i.e., $\mL=\Tilde{\mA}_G$, we have the spectral radius $\rho(\mL)=\max_j \Abs{\lambda_j}=1$, and the spectral norm $\Norm{\mL}_2=1$, according to Lemma \ref{lemma:graph-properties} in Appendix.
The normalized adjacency matrix is defined as 
$
    {\Tilde{\mA}_G} = \Tilde{\mD}^{-\frac{1}{2}} (\mI + \mA_G) \Tilde{\mD}^{-\frac{1}{2}}
$
where $\Tilde{\mD}=\mI+\mD$ such that $\mD=\mathrm{diag}(D_1,\dots,D_n)$ with $D_i$ being the degree of node $i$, and $D_{\max},D_{\min}$ being the maximum and minimum degrees across nodes. 
Let $(\lambda_i, \vv_i)$ be eigenpairs of $\mL$, i.e., $\mL \vv_i = \lambda_i \vv_i$ where self-loops remove negative eigenvalues, i.e., $0  \le \lambda_n\le \dots \le \lambda_1 = 1$, and improve numerical stability.
As far as SGC \cite{wu2019simplifying} is concerned, we have $\phi=\mI$ and $\mL=\mA_G$, such that the product of $\{\mW_l\}$ reduces to a single $\mW$.
 
\textbf{Random Walk GCN}. When we consider the random walk as the graph propagation matrix, i.e., $\mL=\mD^{-1}{\mA}_G$, we have $\mL$ a row-stochastic matrix with the sum of each row being 1. As such, $\vone$ is a right eigenvector of $\mL$ with the corresponding eigenvalue $\lambda_1 = 1$. Therefore, the spectral radius $\rho(\mL)=1$, and the eigenvalues lie in the range of $[-1,1]$. 
A GCN-style variant of graphSAGE \cite{hamilton2017inductive} can be represented as $\mL=\Tilde{\mD}^{-1}(\mI + \mA_G)$, a row-stochastic matrix, such that it performs a row-normalized mean over its closed neighborhood. The difference from the random walk GCN is the self-loops, so that we have the spectral radius $\rho(\mL)=1$, the spectral norm $1 \le \Norm{\mL}_2 \le \sqrt{\frac{D_{\max}+1}{D_{\min}+1}}$ according to Lemma \ref{lemma:graph-properties}, and the eigenvalues lie in the range of $[0,1]$, yielding more stable and smoothing behavior.

\subsection{Norm-based Bounds}
\textbf{Margin loss.} For any distribution $\mathcal{D}$ and margin $\gamma>0$, the expected margin loss is defined as
\begin{align}
    L_{\gamma}(f_{\vw}) \triangleq \mathbb{E}_{z \sim \mathcal{D}} \vone\Big( f_{\vw}(\mX)[y] \le \gamma + \max_{j \neq y} f_{\vw}(\mX)[j] \Big)
\end{align}
where $\vone(\cdot)$ is the indicator function. The empirical margin loss $\hat{L}_{\gamma}(f_{\vw})$ is the estimate of $L_{\gamma}(f_{\vw})$ with the average over the training dataset $\mathcal{S}$. When $\gamma=0$, $L_{0}(f_{\vw})$ and $\hat{L}_{0}(f_{\vw})$ denote the expected risk and the training error, respectively.

\textbf{PAC-Bayesian generalization bounds.}  Given a deterministic classifier $f_{\vw} \in \gH$, consider its randomized counterpart $f_{\vw+\vu}$ with the prior distribution $P$ and the corresponding posterior distribution $Q$, where $\vu$ is a random weight perturbation.
With probability at least $1-\delta$, over the i.i.d. training set $\gS$ of size $m$, we have the ``two-sided'' PAC-generalization error bound \cite{mcallester2003simplified,liao2020pac}
\begin{align}
    \E_{\vu} [L_0(f_{\vw+\vu})] &\le \E_{\vu} [\hat{L}_0(f_{\vw+\vu})] \nn \\
    &+ \sqrt{\frac{\KL(\vw+\vu || P) + \ln \frac{2m}{\delta}}{2(m-1)}},
\end{align}
where $\KL(\cdot || \cdot)$ is the KL divergence of two distributions.
For any $\gamma,\delta > 0$, for any $\vw$ and random perturbation $\vu$ subject to the perturbation condition
\begin{align} \label{eq:pert-condition}
    \mathbb{P}_{\vu} [\max_{\vx \in \gX} \|f_{\vw+\vu}(\vx)-f_{\vw}(\vx)\|_{\infty} < \frac{\gamma}{4} ] \ge \frac{1}{2}
\end{align}
with probability at least $1-\delta$, we have the margin-based PAC-Bayesian generalization error bound \cite{mcallester2003simplified,langford2002pac,liao2020pac}
\begin{align}
    L_0(f_{\vw}) \le \hat{L}_{\gamma}(f_{\vw}) + \sqrt{\frac{2\KL(\vw+\vu || P) + \ln \frac{8m}{\delta}}{2(m-1)}}.
\end{align}

Let the random weight perturbation be $\vu=(\vu_1^T,\dots,\vu_d^T)^T$ where $\vu_l=\mathrm{vec}(\mU_l)$ is the weight perturbation at $l$-th layer with $\mU_l$ being the perturbation added to $\mW_l$.
Let $h = \max_l h_l$ for simplicity.
Assuming $\vu \sim \mathcal{N}(\vzero,\sigma^2 \mI)$, for any $B, d, h > 0$, for a $d$-layer neural network with ReLU activation functions, for any $\gamma, \delta > 0$, with probability at least $1-\delta$, we have the following spectrally-normalized margin-based generalization error bound \cite{neyshabur2018pac,liao2020pac}
\begin{align} \label{eq:gen-bound-origin}
    L_0(f_{\vw}) \le \hat{L}_{\gamma}(f_{\vw})   + \mathcal{O} \left( \sqrt{ \frac{B^2 d^2 h \ln(dh) \Phi(\vw) + \ln \frac{dm}{\delta}}{\gamma^2 m}} \right)
\end{align}
with spectral complexity $\Phi(\vw)=\prod_{l=1}^d \|\mW_l\|_2^2 \sum_{l=1}^d \frac{\|\mW_l\|_F^2}{\|\mW_l\|_2^2}$. 

\textbf{PAC-Bayesian bound for GCN.}
The generalization error bound can be extended to the GCNs \cite{liao2020pac,sun2024pac}, taking graph structure into account,\footnote{In \cite{liao2020pac}, there is an additional factor of $D^{d-1}$ to capture graph structure, which is turned out to be unnecessary as stated in \cite{sun2024pac}. In addition, the factor $\Norm{\mL}_2^{d-1}$ can also be removed, when considering $\mL=\Tilde{\mA}_G$ for the graph convolutional layers as $\Norm{\Tilde{\mA}_G}_2=1$, according to Lemma \ref{lemma:graph-properties}.} i.e.,
\begin{align} \label{eq:gen-bound-gcn}
    \MoveEqLeft L_0(f_{\vw}) \le \hat{L}_{\gamma}(f_{\vw}) \nn \\
    &  + \mathcal{O} \left( \sqrt{ \frac{B^2 d^2 h \ln(dh) \Norm{\mL}_2^{2d-2} \Phi(\vw) + \ln \frac{dm}{\delta}}{\gamma^2 m}} \right).
\end{align}
The generalization bound \eqref{eq:gen-bound-gcn} is built upon a key lemma of the GCN perturbation bound in \cite[Lemma 3.1]{liao2020pac} and \cite[Lemma 4.3]{sun2024pac}, where the change of the output due to weight perturbation is upper bounded as
\begin{align} \label{eq:pert-bound-origin}
    \Norm{f_{\vw+\vu}(\mX)-f_{\vw}(\mX)}_2 \le e B \Norm{\mL}_2^{d-1} \prod_{l=1}^d \|\mW_l\|_2 \sum_{l=1}^d \frac{\|\mU_l\|_2}{\|\mW_l\|_2}
\end{align}
when $\|\mU_l\|_2 \le \frac{1}{d}\|\mW_l\|_2$, for any weight perturbation $\vu$ and any bounded input $\mX$ subject to $\Norm{\mX}_{2,\infty}\le B$. 
Notice that $\prod_{l=1}^d \|\mW_l\|_2$ is the multiplicative chain of Lipschitz constants across layers, and $\sum_{l=1}^d \frac{\|\mU_l\|_2}{\|\mW_l\|_2}$ linearly aggregates perturbation sensitivity via a weighted sum of normalized perturbations.

\begin{figure*}
\begin{subequations} \label{eq:main-opt}
\begin{align} 
    \min_{\sigma^2, \mR_l, \mA_l} \quad & \KL(\sigma^2,\mR_l) \triangleq \frac{1}{2} \sum_{l=1}^d\frac{\|\mW_l\|_F^2}{\sigma^2} + \Tr(\mR_l) - \log \det \mR_l - \dim(\mR_l)\label{eq:main-opt-obj}\\
    \mathrm{s.t.} \quad & \mathbb{P}_{\vu_l \sim \gN(\vzero,\sigma^2 \mR_l)} \Big[\sum_{l=1}^d \Norm{\mA_l \vu_l}_2^2 < \frac{\gamma^2}{16} \Big] \ge \frac{1}{2},\label{eq:perb-cond}\\
    &\|f_{\vw+\vu}(\vx)-f_{\vw}(\vx)\|_{\infty}^2 \le \sum_{l=1}^d\|\mA_l \vu_l\|_2^2, \label{eq:perb-bound}
\end{align}
\end{subequations}
 \hrule
\end{figure*} 

\subsection{A Unified Framework}
To address the issues of isotropically Gaussian distributed weight perturbation and the reliance on the concentration inequality of matrix spectral norm, a unified framework has been proposed in \cite{yi2026unified} by introducing two auxiliary variables.
First, consider an anisotropic weight perturbation, e.g., $\vu \sim \mathcal{N}(0,\sigma^2 \mR)$ to capture different sensitivities to weight perturbations, where $\mR$ possesses a block-diagonal structure, i.e., $\mR = \mathrm{blkdiag}(\mR_1, \dots, \mR_d)$ with each diagonal block $\mR_l$ up to optimization.
Then, introduce a sensitivity matrix $\mA = \mathrm{blkdiag}(\mA_1, \dots, \mA_d)$ with $\mA_l$ being carefully designed for each layer $l$, to encode different network structures, so as to impose different sensitivities on the perturbation to the output, and enable the flexible designs of perturbation bounds.

It is worth noting that, the block-diagonal structure is only imposed on the random weight perturbations $\{\vu_l\}_{l=1}^d$, but not the model weights $\{\vw_l\}_{l=1}^d$ that can be correlated across layers. It is because $\{\vu_l\}_{l=1}^d$ serves as a probabilistic probe of network sensitivity, so that its design can be arbitrary to align with the model's sensitivity geometry.
Such an inductive bias of $\vu$ offers us interpretable generalization bounds, although the block-diagonal structures might render itself less thoroughly probing the network's sensitivity across layers, and possibly yield not sufficiently tight bounds.

In doing so, the derivation of the generalization bound can be done by solving the stochastic optimization problem on the top of the next page.
Note here that the $\ell_\infty$ norm is upper-bounded by the $\ell_2$ norm in such a way that proper concentration inequalities of the $\ell_2$ norm ensure a closed-form solution $\mR_l$ with respect to $\mA_l$. Here $\mA_l$ is referred to as a Jacobian-like sensitivity matrix, which is a function of the weights and the inputs. The goal is to design proper sensitivity matrices to obtain tighter and interpretable generalization bounds.
Directly optimizing \eqref{eq:main-opt} is intractable, as it involves the perturbation condition in a probabilistic form.
Therefore, \cite{yi2026unified} proposed to take the following steps to design $\mA_l$, $\sigma^2$, and $\mR_l$ in parallel.
\subsubsection{Design sensitivity matrix to satisfy perturbation bound} By bounding the output difference with respect to weight perturbation $\{\vu_l\}_{l=1}^d$, we design proper choices of sensitivity matrices $\{\mA_l\}_{l=1}^d$ to satisfy the perturbation bound in \eqref{eq:perb-bound}.
There are two possible ways revealed in \cite{yi2026unified} to find proper choices of $\mA_l$: one is from the first-order Taylor expansion of $f_{\vw+\vu}(\vx)-f_{\vw}(\vx)$ as a function of the Jacobian of the weights and the weight perturbation $\vu$; the other one is to relate it to the perturbation bound derived in \cite{neyshabur2018pac} and widely used in the literature (e.g., \cite{liao2020pac,sun2024pac}). As such, we could end up with various designs of $\{\mA_l\}_{l=1}^d$ as a function of weights $\{\mW_l\}_{l=1}^d$ and some imposed biases related to network structures, e.g., diagonal, residual, low-rank, circulant, and Toeplitz.

\subsubsection{Figure out the prior variance with concentration inequalities}
Next, we can further leverage the concentration inequalities for $\ell_2$ norm of vectors, e.g., the Hanson-Wright inequality \cite{rudelson2013hanson} that bounds the quadratic form of sub-Gaussian random vectors, to figure out the proper choices of $\sigma^2$, so that the perturbation condition in \eqref{eq:perb-cond} is satisfied with probability at least $\frac{1}{2}$.
As such, $\Norm{\sum_l \mA_l \vu_l}_2^2$ can be concentrated on its mean with a deviation term controlled via its Frobenius and $\ell_2$ norms, i.e.,
with probability at least $\frac{1}{2}$, it holds
\begin{align}
     \sum_{l=1}^d \Norm{\mA_l \vu_l}_2^2
     &\le \sigma^2 \kappa \sum_{l=1}^d \Tr(\mA_l \mR_l \mA_l^T) \le \frac{\gamma^2}{16}
\end{align}
where $\kappa=1+2\ln2+\sqrt{4\ln2}$ comes from the upper bounds of its deviation terms of Frobenius and $\ell_2$ norms.
By setting 
\begin{align} \label{eq:sigma-diag}
\frac{1}{\sigma^2} = 
\frac{16 \kappa}{\gamma^2} \sum_{l=1}^d \Tr(\mA_l \mR_l \mA_l^T), 
\end{align}
the original perturbation condition can be satisfied.

However, the prior variance $\sigma^2$ should not depend on the trained weights $\{\mW_l\}_{l=1}^d$, yet the sensitivity matrices $\{\mA_l\}$ are inevitably involving weight matrices. To resolve such an issue, we follow the tricks proposed in \cite{neyshabur2018pac} to consider the spectral-normalized weights, i.e., replacing $\beta=\Norm{\mW_l}_2$ for all $l$ with its estimate $\hat{\beta}=\Norm{\hat{\mW}_l}_2$ subject to $\Abs{\beta-\hat{\beta}} \le \frac{1}{d}\beta$, in such a way that $\sigma^2$ is chosen as a function of $\{\hat{\mW}_l\}_{l=1}^d$ instead.

\subsubsection{Optimize the posterior covariance matrix to minimize the KL divergence}
To figure out the optimal $\mR_l^*$ as a function of $\gamma$ and $\mA_l$, by KKT conditions, we have
\begin{align} \label{eq:opt-Rl}
    \mR_l^* = \left(\mI + \frac{16 \kappa\Norm{\vw}_2^2}{\gamma^2} \mA_l^T \mA_l  \right)^{-1}.
\end{align}

\subsubsection{Upper-bound the KL divergence} Given a fixed $\hat{\beta}$, for all $\beta$ such that $\Abs{\beta-\hat{\beta}} \le \frac{1}{d}\beta$, with the optimized posterior covariance $\{\mR_l\}_{l=1}^d$ and the prior variance $\sigma^2$, we are able to derive an upper bound of the KL divergence term in \eqref{eq:main-opt-obj}, and therefore the generalization error bounds. Finally, by taking a union bound over all chosen $\hat{\beta}$, we end up with the final generalization bound for any weights.

Such a unified framework allows us to probe the sensitivity of neural networks at hand by flexible designs of sensitivity matrices, in such a way that the generalization bounds measure how sensitive certain network architectures are to the generalization performance. This gives us a flexible and powerful tool to inspect DNN's generalization.

\section{PAC-Bayesian Generalization Bounds}

In what follows, we extend this unified framework to graph neural networks for graph classification and obtain two generalization error bounds with spatial and spectral sensitivity matrix designs that account for graph structures.

\begin{theorem}\label{thm:generalization-bounds-spatial}
Consider a GCN-type model $f_{\vw}: \gX \times \gG \to \R^K$ with the graph diffusion matrix $\mL \in \R^{n \times n}$ and ReLU activations, depth $d$, width at most $h$, and bounded inputs $\Norm{\mX}_{2,\infty} \le B$.
For any $\delta,\gamma>0$, with probability at least $1-\delta$, for any $\vw$, the generalization error over an i.i.d. training set of size $m$ can be upper-bounded by

\begin{small}    
    \begin{align}
    L_0(f_{\vw}) \le \hat{L}_{\gamma}(f_{\vw}) + \gO \left( \sqrt{ \frac{ B^2 d^2K \frac{1}{n} \Norm{\mL^{d-1}\vone}_2^2\Phi(\vw) + \ln \frac{dm}{\delta}}{\gamma^2m}} \right)
    \end{align}
\end{small}

where $\Phi(\vw)=\prod_l \Norm{\mW_l}_2^2  \sum_{l=1}^d \frac{\Norm{\mW_l}_F^2}{\Norm{\mW_l}_2^2}$.
\end{theorem}
\begin{proof}
    See Section \ref{proof:generalization-bounds-spatial}.
\end{proof}
\begin{remark} \upshape
Besides the spectral complexity of weight matrices, i.e., $\Phi(\vw)$, we have an additional term $\Norm{{\mL}^{d-1}\vone}_2^2$ to capture the graph structures for generalization. Specifically, 
it measures the total energy of the propagated constant signal over $d$ graph convolutional layers. If it grows with $d$, the graph operator amplifies constant signals, potentially leading to instability or over-smoothing. If it decays, the operator suppresses constant signals, which may help preserve node distinctions.

As $\Norm{{\mL}^{d-1}\vone}_2^2 \le n\Norm{{\mL}^{d-1}}_2^2 \le n\Norm{{\mL}}_2^{2d-2}$ always holds, the generalization bound in Theorem \ref{thm:generalization-bounds-spatial} is strictly tighter than the state-of-the-art ones by at least a factor of $\ln(dh)$. 
Notably, for some specific designs of GCNs, we could obtain strictly much tighter bounds.
For the vanilla GCN, $\mL$ is the normalized adjacency matrix, so that, as in Lemma \ref{lemma:graph-properties}, we have 
\begin{align}
    \left( \frac{\sum_i \sqrt{D_{i}}}{\sqrt{\sum_i D_{i}}}\right)^2 \le \Norm{{\mL}^{d-1}\vone}_2^2 \le n,
\end{align}
where the last equality holds when considering regular graphs that yield the loosest bound with $\Norm{{\mL}^{d-1}\vone}_2^2=n$. When graphs are irregular, we have the tightest bound which is node degree dependent. Therefore, the scaling factor is at most $\gO(d^2K \Phi(\vw))$, which is strictly tighter than $\gO(d^2h\ln(dh) \Phi(\vw))$ in \cite{liao2020pac,sun2024pac}.

For the random walk GCN with self-loops, as $\mL\vone=\vone$, we have $\mL^{d-1}\vone=\vone$, which is irrelevant to $d$. That is, $\Norm{{\mL}^{d-1}\vone}_2^2=n$ for any $d$. 
It implies that, regardless of the depth $d$, the sum of the incoming ``influence'' to each node remains uniformly constant at 1, indicating a measure-preserving diffusion. As such, the generalization bound scaling as $\gO(d^2K \Phi(\vw))$, which is still tighter than the existing bounds in \cite{liao2020pac,sun2024pac} with scaling factor of $\gO(d^2h \ln(dh) \Norm{\mL}_2^{2d-2} \Phi(\vw))$, given that $1 \le \Norm{\mL}_2 \le \sqrt{\frac{D_{\max}+1}{D_{\min}+1}}$ according to Lemma \ref{lemma:graph-properties}.

\end{remark}

\begin{theorem}\label{thm:generalization-bounds-spectral}
Consider a GCN-type model $f_{\vw}: \gX \times \gG \to \R^K$ with ReLU activations, graph spectra $\{\lambda_i\}_{i=1}^n$, depth $d$, width at most $h$, and bounded inputs $\Norm{\mX}_{2,\infty} \le B$.
For any $\delta,\gamma>0$, with probability at least $1-\delta$, for any $\vw$, the generalization error over an i.i.d. training set with size $m$ can be upper-bounded by
    \begin{align}
    L_0(f_{\vw}) \le \hat{L}_{\gamma}(f_{\vw}) + \gO \left( \sqrt{ \frac{ B^2 d^2h \Phi^{\mathrm{sp}}(\vw) + \ln \frac{dm}{\delta}}{\gamma^2m}} \right)
    \end{align}
    with the graph spectral complexity $\Phi^{\mathrm{sp}}(\vw) = \prod_{l=1}^d \Norm{{\mW}_l}_2^2 \sum_{l=1}^d g_l^2(\lambda) \frac{\Norm{{\mW}_l}_F^2}{\Norm{{\mW}_l}_2^2}$ and the spectral sensitivity function $g_l(\lambda)=\max_i \Abs{\lambda_i^{d-l-1}} \max_j \Abs{\psi(\lambda_j)\lambda_j^{l-1}}$ with $\psi(\cdot)$ being graph filters up to design.
\end{theorem}

\begin{proof}
    See Section \ref{proof:generalization-bounds-spectral}.
\end{proof}

\begin{remark}\upshape
    In terms of scaling laws, compared with the generalization bounds in \cite{liao2020pac,sun2024pac}, i.e., $\gO(d^2h\ln(dh) \Phi(\vw))$, Theorem \ref{thm:generalization-bounds-spectral} has an order of $\gO(d^2h \Phi^{\mathrm{sp}}(\vw))$, which is smaller with respect to $d$ and $h$, yet $\Phi^{\mathrm{sp}}(\vw)$ depends on the choices of the spectral sensitivity function $g_l(\lambda)$. In doing so, the generalization bounds are topology-aware in the sense that different designs of graph filters help us probe network sensitivities to inspect possible over-smoothing or over-squashing for homophilic or heterophilic graphs.  
    By choosing the graph filter $\psi(\lambda)$ to be an identity, a low-pass and a high-pass filters, we have different designs of $g_l(\lambda)$ with the same scaling as $\rho(\mL)^{d-1}$, which is similar to those in \cite{liao2020pac,sun2024pac}. Nevertheless, $g_l(\lambda)$ has different spectral behaviors against the graph spectra $\lambda$, in such a way that $\Phi^{\mathrm{sp}}(\vw))$ reflects different generalization behavior for different graph neural networks with over-smoothing/over-squashing and/or homophily/heterophily properties.
\end{remark}

\section{Proofs}
To prove Theorems \ref{thm:generalization-bounds-spatial} and \ref{thm:generalization-bounds-spectral}, we start with the derivation of the Jacobian matrix, leveraging the first-order Taylor expansion of network outputs as a function of the Jacobian of the weights in the unified framework.

Define $\mJ_l = \frac{\partial f_{\vw}}{\partial \mathrm{vec}(\mW_l)}$ as the Jacobian matrix of the network output with respect to the weight matrix at the $l$-th layer.
For the readout layer, we have
\begin{align}
    \mJ_d =  \frac{1}{n} \mI_{h_d} \otimes (\vone^T \mH_{d-1} ) \in \R^{h_d \times h_{d-1}h_d}.
\end{align}
Define $\mG_l=\frac{\partial f}{\partial \vec(\mH_l)} \in \R^{K \times nh_l}$ as the Jacobian matrix of the output with respect to the graph embeddings at the $l$-th layer. According to Lemma \ref{lemma:jacobian-matrix} in Appendix, for the graph convolutional layer-$l$, we have
\begin{align} \label{eq:jacobian-matrix}
    \mJ_l = \mG_l \mB_l \big(\mI_{h_l} \otimes ({\mL}\mH_{l-1})\big)
\end{align}
with $\mH_{l-1}=\phi({\mL} \mH_{l-2} \mW_{l-1})$, where 
\begin{align} \label{eq:G_l}
   \mG_l = \frac{1}{n}\big(\mW_d^T \otimes \vone^T\big) \left(\prod_{k=d-1}^{l+1} \mB_k(\mW_k^T \otimes {\mL})\right)
\end{align}
and the diagonal matrix $\mB_l=\mathrm{diag}(\mathrm{vec}(\phi'({\mL} \mH_{l-1}\mW_l)) \in \R^{nh_l \times nh_l}$, with $\phi'$ being the derivative of the activation function. When $\phi$ is ReLU, then the diagonal elements in $\mB_l$ is either 0 or 1, yielding $\mB_l \preceq \mI$.
Therefore we have
\begin{align} \label{eq:jacobian-jacobian}
    \mJ_l^T \mJ_l &= \big(\mI_{h_l} \otimes (\mH_{l-1}^T{\mL})\big) \mB_l \mG_l^T \mG_l \mB_l \big(\mI_{h_l} \otimes ({\mL}\mH_{l-1})\big).
\end{align}

In what follows, we present two types of sensitivity matrix design approaches. One is from a spatial perspective to consider a similar design of $\mA_l$ as the Jacobian matrix $\mJ_l$, and the other is from a spectral viewpoint to consider different graph filters.

\subsection{Spatial Designs and Proof of Theorem \ref{thm:generalization-bounds-spatial}}
\label{proof:generalization-bounds-spatial}
\subsubsection{Perturbation bound}
For the Jacobian matrix in \eqref{eq:jacobian-jacobian}, given the fact that $\mB_l \preceq \mI$, we have 
\begin{align}
    \mJ_l^T \mJ_l &= 
    \Big(\mI_{h_l} \otimes  ( \mH_{l-1}^T {\mL}^T)\Big) \mB_l \left(\prod_{k=l+1}^{d-1} (\mW_k \otimes {\mL}^T)\mB_k^T\right) \nn \\
    & \qquad \cdot \frac{1}{n^2}\big(\mW_d \mW_d^T \otimes \vone\vone^T\big)\left(\prod_{k=d-1}^{l+1} \mB_k(\mW_k^T \otimes {\mL})\right)\nn\\
    & \qquad \qquad \qquad\qquad \cdot\mB_l \Big(\mI_{h_l} \otimes  ({\mL} \mH_{l-1})\Big)\\
    &\preceq \frac{1}{n^2} \big( \prod_{k=d}^{l+1} \mW_k^T \big) \big( \prod_{k=l+1}^{d} \mW_k \big) \nn \\
    &\qquad \qquad \otimes \mH_{l-1}^T {\mL}^{d-l} \vone \vone^T {\mL}^{d-l} \mH_{l-1}\\
    & \preceq \frac{1}{n^2} \big( \prod_{k=d}^{l+1} \mW_k^T \big) \big( \prod_{k=l+1}^{d} \mW_k \big) \nn \\
    &\qquad \otimes \prod_{k=l-1}^1 \mW_k^T \mX^T {\mL}^{d-1} \vone \vone^T {\mL}^{d-1} \mX \prod_{k=1}^{l-1} \mW_k
\end{align}
Let
\begin{align}
   \mA_l &= \frac{\sqrt{d}}{n}\big( \prod_{k=d}^{l+1} \mW_k^T \big) \otimes (\vone^T {\mL}^{d-1} \mX \prod_{k=1}^{l-1} \mW_k).
\end{align}
It is readily to verify that $\mJ_l^T \mJ_l \preceq \frac{1}{d}\mA_l^T \mA_l$, where
\begin{align}
    \mA_l^T \mA_l &= \mN_l \mN_l^T \otimes \vv_l \vv_l^T\\
     \mA_l \mA_l^T &=  \Norm{\vv_l}_2^2 \mN_l^T \mN_l  
\end{align}
with $\mN_l^T = \frac{\sqrt{d}}{n}\prod_{k=d}^{l+1} \mW_k^T$ and $\vv_l^T=\vone^T {\mL}^{d-1} \mX \prod_{k=1}^{l-1} \mW_k$.
According to the first-order Taylor expansion, we have
\begin{align}
    \Norm{f_{\vw+\vu}(\mX)-f_{\vw}(\mX)}_2^2 &= \Norm{\sum_{l=1}^d \mJ_l \vu_l}_2^2 + o(\Norm{\vu_l}_2^2)\\
    &\le d \sum_{l=1}^d \Norm{ \mJ_l \vu_l}_2^2 \\
    &\le \sum_{l=1}^d \Norm{ \mA_l \vu_l}_2^2.
\end{align}
Therefore, the perturbation bound is satisfied since $\Norm{\cdot}_{\infty} \le \Norm{\cdot}_2$ for vectors. 
As a byproduct, by the first-order Taylor expansion, we are able to obtain a new GCN perturbation bound in Lemma \ref{lemma:GCN-pert-bound-new} comparable as in \cite{liao2020pac} with relaxed assumptions.

To make the KL term more tractable, we make  the following two specific designs of $\mA_l$.

\subsubsection{Sensitivity matrices}
In what follows, we consider two designs of Jacobian-type sensitivity matrices.

\textbf{Diagonal case.}
For simplicity, we can set the sensitivity matrix $\mA_l$ as a diagonal matrix, i.e.,
\begin{align}
    \mA_l^{\mathrm{diag}} = \sqrt{\frac{d}{n}} B \Norm{{\mL}^{d-1}\vone}_2 \prod_{i \ne l} \Norm{\mW_i}_2 \mI_{h_l h_{l-1}}
\end{align}
so that $\sum_{l=1}^d \Norm{ \mA_l \vu_l}_2^2 \le \sum_{l=1}^d \Norm{ \mA_l^{\mathrm{diag}} \vu_l}_2^2$ for any $\vu_l$ due to $\mA_l^T \mA_l \preceq (\mA_l^{\mathrm{diag}})^T \mA_l^{\mathrm{diag}}$ resulted from $\Norm{\mX}_2 \le \Norm{\mX}_F \le \sqrt{n}\Norm{\mX}_{2,\infty} \le \sqrt{n}B$ and the triangle inequalities of the operator norms.

\textbf{Low-rank case.} As $\mA_l$ is a fat matrix with usually $K=h_d \le h_l h_{l-1}$, we define an explicit rank-$K$ sensitivity matrix, i.e.,
\begin{align}
\mA_l^{\mathrm{lr}}=\frac{\sqrt{d}}{n} \big(\prod_{k=d}^{l+1} \Norm{\mW_k}_2 \mI_{K} \big) \otimes \big(\vone^T {\mL}^{d-1} \mX \prod_{k=1}^{l-1} \mW_k \big)
\end{align}
so that $\sum_{l=1}^d \Norm{ \mA_l \vu_l}_2^2 \le \sum_{l=1}^d \Norm{ \mA_l^{\mathrm{lr}} \vu_l}_2^2$ for any $\vu_l$ due to $\mA_l^T \mA_l \preceq (\mA_l^{\mathrm{lr}})^T \mA_l^{\mathrm{lr}}$ resulted from Lemma \ref{lemma:kron-properties} and the fact that $\mN_l^T \mN_l \preceq \Norm{\mN_l}_2^2 \mI_{K}$ is equivalent to $\mN_l \mN_l^T \preceq \Norm{\mN_l}_2^2 \mI_{h_l}$.
It also holds that 
\begin{align}
\mA_l \mA_l^T =  \Norm{\vv_l}_2^2 \mN_l^T \mN_l \preceq \Norm{\vv_l}^2 \Norm{\mN_l}_2^2 \mI_{K} = \mA_l^{\mathrm{lr}} (\mA_l^{\mathrm{lr}})^T.
\end{align}

\subsubsection{Prior variance} In general, we need to choose a proper $\sigma^2$ as a function of $\gamma$, $\mA_l$ and $\mR_l$ to satisfy the perturbation condition \eqref{eq:perb-cond} with the following chain of inequalities
\begin{align}
   \sum_{l=1}^d \Norm{\mA_l \vu_l}_2^2 \le \sigma^2 \kappa \sum_{l=1}^d \Tr(\mA_l \mR_l \mA_l^T) \le \frac{\gamma^2}{16}. \label{eq:chain-of-ineq}
\end{align}
As long as there is some $\sigma^2$ to make the above inequalities satisfied, the perturbation condition holds.

As $\sigma^2$ also appears in the prior distribution, which should not depend on training weights $\vw$, we introduce the estimates of weight matrices as auxiliary variables leveraging the spectral normalization tricks in \cite{neyshabur2018pac,liao2020pac}. 

\textbf{Diagonal case.}
Specifically, for the diagonal design, let
\begin{align}
    \hat{\mA}_l^{\mathrm{diag}} = \sqrt{\frac{d}{n}} B \Norm{{\mL}^{d-1}\vone}_2 \prod_{i \ne l} \Norm{\hat{\mW}_i}_2 \mI_{h_l h_{l-1}}.
\end{align}
Given $\beta=\Norm{{\mW}_l}_2$ and $\hat{\beta}=\Norm{\hat{\mW}_l}_2$, for $\Abs{\beta-\hat{\beta}} \le \frac{1}{d} \beta$, we have $\frac{1}{e} \beta^{d-1} \le \hat{\beta}^{d-1} \le e \beta^{d-1}$, such that
\begin{align} \label{eq:AA-chain-inequalities}
    \frac{1}{e^2} \mA_l^{\mathrm{diag}} (\mA_l^{\mathrm{diag}})^T \preceq \hat{\mA}_l^{\mathrm{diag}} (\hat{\mA}_l^{\mathrm{diag}})^T \preceq {e^2} \mA_l^{\mathrm{diag}} (\mA_l^{\mathrm{diag}})^T.
\end{align}
Together with the fact that 
\begin{align}
    \mA_l \mR_l \mA_l^T \preceq \mA_l \mA_l^T \preceq \mA_l^{\mathrm{diag}} (\mA_l^{\mathrm{diag}})^T
\end{align}
as $\mR_l \preceq \mI$, we rewrite the perturbation condition as
\begin{align}
   \sum_{l=1}^d \Norm{\mA_l \vu_l}_2^2 &\le \sigma^2 \kappa \sum_{l=1}^d \Tr(\mA_l \mR_l \mA_l^T) \nn \\
   &\le e^2 \sigma^2 \kappa \sum_{l=1}^d \Tr(\hat{\mA}_l^{\mathrm{diag}} (\hat{\mA}_l^{\mathrm{diag}})^T) \le \frac{\gamma^2}{16}, 
\end{align}
which yields a valid choice of $\sigma^2$, i.e,
\begin{align}
    \frac{1}{\sigma^2} = \frac{16e^2\kappa}{\gamma^2} \sum_{l=1}^d \Tr(\hat{\mA}_l^{\mathrm{diag}} (\hat{\mA}_l^{\mathrm{diag}})^T)
\end{align}
as a function of weight estimates $\{\hat{\mW}_l\}_{l=1}^d$.

\textbf{Low-rank case.}
Given $\beta=\Norm{{\mW}_l}_2$ and $\hat{\beta}=\Norm{\hat{\mW}_l}_2$, for $\Abs{\beta-\hat{\beta}} \le \frac{1}{d} \beta$, we have $\frac{1}{e} \beta^{d-1} \le \hat{\beta}^{d-1} \le e \beta^{d-1}$, such that
\begin{align}
    {\mA}_l^{\mathrm{lr}} ({\mA}_l^{\mathrm{lr}})^T &= \frac{d}{n^2} \Norm{\vv_l}_2^2  \prod_{k=d}^{l+1} \Norm{{\mW}_k}_2^2 \mI_{K} \\
    & \preceq \frac{B^2d}{n} \Norm{{\mL}^{d-1}\vone}_2^2 \prod_{k \ne l} \Norm{{\mW}_k}_2^2 \mI_{K}\\
    & \preceq \frac{e^2B^2d}{n} \Norm{{\mL}^{d-1}\vone}_2^2 \prod_{k \ne l} \Norm{\hat{\mW}_k}_2^2 \mI_{K}
\end{align}
with $\hat{\vv}_l^T=\vone^T {\mL}^{d-1} \mX \prod_{k=1}^{l-1} \hat{\mW}_k$.
Together with the fact that 
\begin{align}
    \mA_l \mR_l \mA_l^T \preceq \mA_l \mA_l^T \preceq \mA_l^{\mathrm{lr}} (\mA_l^{\mathrm{lr}})^T
\end{align}
as $\mR_l \preceq \mI$, we rewrite the perturbation condition as
\begin{align}
   \sum_{l=1}^d \Norm{\mA_l \vu_l}_2^2 &\le \sigma^2 \kappa \sum_{l=1}^d \Tr(\mA_l \mR_l \mA_l^T) \nn \\
   &\le \sigma^2 \kappa \sum_{l=1}^d \Tr({\mA}_l^{\mathrm{lr}} ({\mA}_l^{\mathrm{lr}})^T)\\
   & \le \frac{e^2 B^2 d K \sigma^2 \kappa}{n} \Norm{{\mL}^{d-1}\vone}_2^2 \sum_{l=1}^d \prod_{k \ne l} \Norm{\hat{\mW}_k}_2^2
   \le \frac{\gamma^2}{16}, 
\end{align}
which yields a valid choice of $\sigma^2$, i.e,
\begin{align}
    \frac{1}{\sigma^2} =  \frac{16e^2 \kappa B^2 d K }{\gamma^2 n} \Norm{{\mL}^{d-1}\vone}_2^2 \sum_{l=1}^d\prod_{k \ne l} \Norm{\hat{\mW}_k}_2^2,
\end{align}
relaxing the dependent of weight matrices.

\subsubsection{Generalization bounds} 
In what follows, we consider the diagonal and low-rank cases separately.

\textbf{Diagonal case.}
With the adjusted ${\sigma}^2$ for a fixed $\hat{\beta}$
and the optimized $\mR_l^*$, i.e.,
\begin{align} \label{eq:opt-Rl}
    \mR_l^* = \left(\mI + \frac{16 \kappa\Norm{\vw}_2^2}{\gamma^2} (\mA_l^{\mathrm{diag}})^T \mA_l^{\mathrm{diag}}  \right)^{-1},
\end{align}
we can obtain the resulting upper bound of the KL divergence in \eqref{eq:main-opt-obj} for $h=\max_l h_l$, i.e.,
\begin{align}
     \KL &=  \frac{1}{2} \sum_{l=1}^d \frac{\|\mW_l\|_F^2}{{\sigma}^2} + \Tr(\mR_l^*) - \log \det \mR_l^* - h^2 \\
     &\le \frac{8e^2\kappa \Norm{\vw}_2^2}{\gamma^2} \sum_{l=1}^d \Tr({\hat{\mA}}_l^{\mathrm{diag}} ({\hat{\mA}}_l^{\mathrm{diag}})^T) \nn \\
     & \qquad \qquad \qquad \qquad + \frac{\eta^2}{2} \Tr((\mA_l^{\mathrm{diag}})^T \mA_l^{\mathrm{diag}})\\
     &\le \frac{8(e^4+1)\kappa \Norm{\vw}_2^2}{\gamma^2} \sum_{l=1}^d \Tr({\mA}_l^{\mathrm{diag}} ({\mA}_l^{\mathrm{diag}})^T)\\
     &\lesssim \gO\left( \frac{B^2 d^2h^2 \beta^{2d-2}}{\gamma^2 n} \Norm{\vw}_2^2 \Norm{{\mL}^{d-1}\vone}_2^2 \right)\\
     &\lesssim \gO\left( \frac{B^2 d^2h^2 \Norm{{\mL}^{d-1}\vone}_2^2}{\gamma^2 n} \Phi(\vw)  \right)
\end{align}
where the first inequality is due to Lemma \ref{lemma:fun-properties} presented in the Appendix, the second inequality comes from \eqref{eq:AA-chain-inequalities}, and the last inequity is with
\begin{align}
    \Phi(\vw)&=\prod_l \Norm{\mW_l}_2^2  \sum_{l=1}^d \frac{\Norm{\mW_l}_F^2}{\Norm{\mW_l}_2^2}.
\end{align}
As such, we end up with new generalization error bounds
\begin{align} \label{eq:gene-bounds-diag}
    \MoveEqLeft L_0(f_{\vw}) \le \hat{L}_{\gamma}(f_{\vw}) \nn \\
    & + \gO \left( \sqrt{ \frac{ B^2 d^2h^2 \frac{1}{n}\Norm{{\mL}^{d-1}\vone}_2^2  \Phi(\vw) + \ln \frac{dm}{\delta}}{\gamma^2m}} \right)
\end{align}
by taking a union bound over all possible $\hat{\beta}$.

\textbf{Low-rank case.}
Given the optimized $\mR_l^*$, i.e.,
\begin{align} \label{eq:opt-Rl}
    \mR_l^* = \left(\mI + \frac{16 \kappa\Norm{\vw}_2^2}{\gamma^2} (\mA_l^{\mathrm{lr}})^T \mA_l^{\mathrm{lr}}  \right)^{-1},
\end{align}
together with the adjusted ${\sigma}^2$ for a fixed $\hat{\beta}$, we can obtain the resulting KL divergence and its upper bounds with $h=\max_l h_l$ for simplicity, i.e.,
\begin{align}
     \KL &=  \frac{1}{2} \sum_{l=1}^d \frac{\|\mW_l\|_F^2}{{\sigma}^2} + \Tr(\mR_l^*) - \log \det \mR_l^* - h^2 \\
     & \le \frac{8e^2 \kappa B^2 d K }{\gamma^2 n} \Norm{\vw}_2^2 \Norm{{\mL}^{d-1}\vone}_2^2 \sum_{l=1}^d\prod_{k \ne l} \Norm{\hat{\mW}_k}_2^2 \nn \\
     & \qquad \qquad \qquad + \frac{\eta^2}{2} \Tr ((\mA_l^{\mathrm{lr}})^T \mA_l^{\mathrm{lr}})\\
     & \le \frac{8(e^4+1) \kappa B^2 d K }{\gamma^2 n} \Norm{\vw}_2^2 \Norm{{\mL}^{d-1}\vone}_2^2 \sum_{l=1}^d\prod_{k \ne l} \Norm{{\mW}_k}_2^2\\
     &\lesssim \gO\left( \frac{B^2 d^2K \Norm{{\mL}^{d-1}\vone}_2^2}{\gamma^2 n} \Phi(\vw)  \right)
\end{align}
by which we have the similar generalization bound as in \eqref{eq:gene-bounds-diag} for any weights as
\begin{align} \label{eq:gene-bounds-low-rank}
    \MoveEqLeft L_0(f_{\vw}) \le \hat{L}_{\gamma}(f_{\vw}) \nn \\
    & + \gO \left( \sqrt{ \frac{ B^2 d^2K \frac{1}{n}\Norm{{\mL}^{d-1}\vone}_2^2  \Phi(\vw) + \ln \frac{dm}{\delta}}{\gamma^2m}} \right)
\end{align}
It is worth noting that when considering the low-rank design of sensitivity matrix, the full dimension of $h^2$ is reduced to $K$, such that $K \ll h^2$ in practice implies a tighter bound.

\subsection{Spectral Designs and Proof of Theorem \ref{thm:generalization-bounds-spectral}}
\label{proof:generalization-bounds-spectral}
Given the Laplacian eigenpairs $(\lambda_i,\vv_i)$ of the graph propagation matrix $\mL$, i.e., $\mL=\mV \mLambda \mV^T$, we can design the sensitivity matrices according to the eigenspace representations.
Specifically, let
\begin{align} \label{eq:spectral-Al}
    {\mA}_l= \alpha_l \big(\mI_{h_l} \otimes (\mV \psi(\mLambda) \mV^T\mH_{l-1})\big)
\end{align}
with 
\begin{align}
    \alpha_l = \frac{\sqrt{d}}{n} \Norm{\vone^T\mL^{d-l-1}}_2 \prod_{k=d}^{l+1} \Norm{\mW_k}_2.
\end{align}
where $\psi(\mLambda)=\mathrm{diag}(\lambda_1,\dots,\lambda_n)$ is the spectral filter, and $\psi:\R \to \R$ is a function (e.g., polynomial, exponential) up to design. For notational simplicity, we denote $\psi(\mL)=\mV f(\mLambda) \mV^T$, such that $\psi(\mL)^l=\mV \psi(\mLambda)^l \mV^T$ for any $l$. It is readily to verify that, if we ensure $\Abs{\psi(\lambda)}  \ge \Abs{\lambda}$ for all $\lambda$, we have $\psi(\mL) \succeq \mL$.
As such, the goal of design is to find proper graph filters $f$ to make sure the perturbation bound is satisfied, and at the same time to be as small as possible to yield a tight generalization bound.

\subsubsection{Perturbation bound}
Given \eqref{eq:G_l}, it follows $ \alpha_l \ge \sqrt{d}\Norm{\mG_l}_2$ for all $l$, since
\begin{align}
    \Norm{\mG_l}_2^2 &\le \frac{1}{n^2}\Norm{(\prod_{k=d}^{l+1} \mW_k^T) \otimes (\vone^T \mL^{d-l-1}) }_2^2\\
    &\le \frac{1}{n^2}\prod_{k=d}^{l+1} \Norm{\mW_k}_2^2 \Norm{\vone^T \mL^{d-l-1}) }_2^2.
\end{align}
Due to the fact that $\mB_l^T \mB_l = \mB_l$, we have $\mB_l \mG_l^T \mG_l \mB_l \preceq \Norm{\mG_l}_2^2 \mI \preceq \frac{\alpha_l^2}{d} \mI$.
Further, if we ensure $\Abs{\psi(\lambda)}  \ge \Abs{\lambda}$ for all $\lambda$, it follows that $\mJ_l^T \mJ_l \preceq \frac{1}{d}{\mA}_l^T {\mA}_l$.  
Therefore, we have
\begin{align}
    \Norm{f_{\vw+\vu}(\mX)-f_{\vw}(\mX)}_2^2 &= \Norm{\sum_{l=1}^d \mJ_l \vu_l}_2^2 + o(\Norm{\vu_l}_2^2)\\
    &\le d \sum_{l=1}^d \Norm{ \mJ_l \vu_l}_2^2 \le \sum_{l=1}^d \Norm{ \mA_l \vu_l}_2^2
\end{align}
such that the perturbation bound is satisfied.

\subsubsection{Prior variance} 
Given the spectral design of the sensitivity matrix in \eqref{eq:spectral-Al}, we have
\begin{align}
     \Tr({\mA}_l^T {\mA}_l) &= {\alpha}_l^2 h_l  \Tr({\mH}_{l-1}^T \mV \psi(\mLambda)^2 \mV^T {\mH}_{l-1})\big)\\
     &= {\alpha}_l^2 h_l \Norm{\psi(\mL){\mH}_{l-1}}_F^2\\
     &\le \frac{d h_l}{n^2} \Norm{\vone^T\mL^{d-l-1}}_2^2 \prod_{k=d}^{l+1} \Norm{{\mW}_k}_2^2\nn \\
     & \qquad \cdot  \Norm{\psi(\mL)\mL^{l-1} \mX}_F^2 \prod_{k=1}^{l-1} \Norm{{\mW}_k}_2^2\\
     &\le \frac{d h_l}{n^2} \Norm{\mX}_F^2\Norm{\vone^T\mL^{d-l-1}}_2^2\nn \\ 
     &\qquad \cdot\Norm{\psi(\mL)\mL^{l-1}}_2^2 \prod_{k\neq l} \Norm{{\mW}_k}_2^2\\
     &\le dh_l B^2 g_l^2(\lambda) \prod_{k\neq l} \Norm{{\mW}_k}_2^2
\end{align}
where $\Norm{\mX}_F \le \sqrt{n}\Norm{\mX}_{2,\infty}\le\sqrt{n}B$, $\Norm{\vone^T\mL^{d-l-1}}_2^2 \le n \rho(\mL)^{2d-2l-2}$ with the spectral radius defined as $\rho(\mL)=\max_j \Abs{\lambda_j}$, and the spectral sensitivity function
\begin{align}
    g_l(\lambda)=\max_i \Abs{\lambda_i^{d-l-1}} \max_j \Abs{\psi(\lambda_j)\lambda_j^{l-1}},
\end{align}
where the first factor $\rho(\mL)^{2d-2l-2}$ is the  worst-case amplification of a constant signal over $d-l-1$ hops, and the second factor represents the strongest single-step amplification possible from the filtered operator $\psi(\mL)\mL^{l-1}$.

In general, we need to choose a proper $\sigma^2$ as a function of $\gamma$, $\mA_l$ and $\mR_l$ to satisfy the perturbation condition \eqref{eq:perb-cond}. 
As $\sigma^2$ also appears in the prior distribution, which should not depend on training weights $\vw$, we introduce an auxiliary sensitivity matrix leveraging the spectral normalization tricks in \cite{neyshabur2018pac,liao2020pac}.

For simplicity, we assume $h=\max_l h_l$.
For the spectral design of the sensitivity matrix, we further have
\begin{align}
     \Tr({\mA}_l^T {\mA}_l) 
     &\le e^2 dh B^2 g_l^2(\lambda) \prod_{k\neq l} \Norm{\hat{\mW}_k}_2^2
\end{align}
due to the fact that $\beta^{d-1} \le e\hat{\beta}^{d-1}$ for $\Abs{\beta-\hat{\beta}}\le \frac{1}{d}\beta$ with $\beta=\Norm{{\mW}_k}_2$ and $\hat{\beta}=\Norm{\hat{\mW}_k}_2$.

According to the perturbation condition, we need to choose a proper $\sigma^2$ to ensure
the following inequalities
\begin{align}
   \sum_{l=1}^d \Norm{\mA_l \vu_l}_2^2 &\le \sigma^2 \kappa \sum_{l=1}^d \Tr(\mA_l \mR_l \mA_l^T) \le \sigma^2 \kappa \sum_{l=1}^d \Tr(\mA_l \mA_l^T) \\
   &\le   \sigma^2 \kappa e^2 dh B^2 \sum_{l=1}^d g_l^2(\lambda) \prod_{k\neq l} \Norm{\hat{\mW}_k}_2^2
   \le \frac{\gamma^2}{16}
\end{align}
where $\kappa=1+2\ln 2 + \sqrt{4\ln 2}$, and $\mR_l = (\mI + \eta^2 \mA_l^T \mA_l)^{-1}$ according to our previous work \cite{yi2026unified}.
As long as there is some $\sigma^2$ to make the above inequalities satisfied, the perturbation condition holds.

By simply setting
\begin{align}
    \frac{1}{\sigma^2} = \frac{ 16 e^2 \kappa d h B^2}{\gamma^2} \sum_{l=1}^d g_l^2(\lambda) \prod_{k\neq l} \Norm{\hat{\mW}_k}_2^2
\end{align}
we have a proper choice of $\sigma^2$ independent of model weights to make the perturbation condition satisfied.

\subsubsection{Generalization bounds} 
According to different graph filter designs and the corresponding choices of $\sigma^2$, we can derive the different generalization bounds. 

With the adjusted ${\sigma}^2$
and the optimized $\mR_l^*$, i.e.,
\begin{align} 
    \mR_l^* = \left(\mI + \frac{16 \kappa\Norm{\vw}_2^2}{\gamma^2} \mA_l^T \mA_l  \right)^{-1},
\end{align}
we can obtain the resulting KL divergence in \eqref{eq:main-opt-obj} or its upper bounds, i.e.,
\begin{align}
     \KL &=  \frac{1}{2} \sum_{l=1}^d \frac{\|\mW_l\|_F^2}{{\sigma}^2} + \Tr(\mR_l^*) - \log \det \mR_l^* - h^2 \\
     & \le \frac{ 8 e^2 \kappa d h B^2}{\gamma^2} \Norm{\vw}_2^2 \sum_{l=1}^d g_l(\lambda) \prod_{k\neq l} \Norm{\hat{\mW}_k}_2^2 \nn \\
     & \qquad \qquad \qquad + \frac{\eta^2}{2} \Tr (\mA_l^T \mA_l)\\
     & \le \frac{8(e^4+1) \kappa d h B^2}{\gamma^2} \Norm{\vw}_2^2 \sum_{l=1}^d g_l^2(\lambda) \prod_{k\neq l} \Norm{{\mW}_k}_2^2\\
     &\lesssim \gO\left( \frac{B^2 d^2h} {\gamma^2} \Phi^{\mathrm{sp}}(\vw)  \right)
\end{align}
where the graph spectral complexity is defined as
\begin{align} \label{eq:graph-spectral-complexity}
    \Phi^{\mathrm{sp}}(\vw) = \prod_{l=1}^d \Norm{{\mW}_l}_2^2 \sum_{l=1}^d g_l^2(\lambda) \frac{\Norm{{\mW}_l}_F^2}{\Norm{{\mW}_l}_2^2}
\end{align}
that involves both spectral complexities of the weight matrices and the graph propagation matrix.
The first term in \eqref{eq:graph-spectral-complexity} measures the worst-case Lipschitz constant of the entire $d$ graph convolutional layers. If any layer has a large weight norm, the entire product blows up, indicating high sensitivity to input perturbations. In the second term, $g_l(\lambda)$ is the graph-induced sensitivity at layer $l$, and $\frac{\Norm{{\mW}_l}_F^2}{\Norm{{\mW}_l}_2^2}$ is effective number of parameters in layer $l$. This sum captures the total graph-weighted complexity across all layers. It turns out that layers with high graph sensitivity ($g_l$) and many effective parameters contribute more to the graph spectral complexity.

Therefore, we end up with a new generalization bound
\begin{align} 
    L_0(f_{\vw}) \le \hat{L}_{\gamma}(f_{\vw}) + \gO \left( \sqrt{ \frac{B^2 d^2h \Phi^{\mathrm{sp}}(\vw) + \ln \frac{dm}{\delta}}{\gamma^2m}} \right)
\end{align}
by taking the union bound over all possible choices of $\hat{\beta}$.
Compared with the previous bounds both in this paper and the existing works, the key factor is the new graph spectral complexity term $\Phi^{\mathrm{sp}}(\vw)$, which is a function of both the weight matrices $\{\mW_l\}_{l=1}^d$ and the graph propagation matrix $\mL$.

\subsubsection{Sensitivity matrices}
The most critical design in the sensitivity matrix is the graph filters $\psi(\lambda)$ for capturing the graph structures in generalization subject to the condition that $\Abs{\psi(\lambda)}  \ge \Abs{\lambda}$ for all $\lambda$ to satisfy the perturbation condition.
With different designs of the sensitivity graph filters $\psi(\lambda)$, we end up with different generalization bounds.
In the following, we focus on different designs of the graph filters $\psi(\lambda)$.

\textbf{Identity filter.} When $\psi(\lambda)=\lambda$, i.e., $\psi(\mL)=\mL$, we have
\begin{align}
    g_l^{\mathrm{id}}(\lambda)=\max_i \Abs{\lambda_i^{d-l-1}} \max_j \Abs{\lambda_j^{l}}
\end{align}
which can be upper bounded by $g_l^{\mathrm{id}}(\lambda) \le \rho(\mL)^{d-1}$. 
This recovers the generalization bounds in \cite{liao2020pac,sun2024pac} with $\rho(\mL)=1$ for the normalized adjacency matrix $\mL$.

The identity filter amplifies all frequencies equally. For smooth graphs (large $\lambda \approx 1$), $g_l(\lambda)$ grows with $l$, indicating depth-dependent sensitivity. For graphs with negative eigenvalues (heterophily), the maximum absolute value could come from $\lambda \approx -1$, leading to oscillatory behavior with layer.

\textbf{Low-pass filter.} When $\psi(\lambda)=\frac{1+2\xi}{1+\xi(1-\lambda)}$ with $\xi>0$, it probes sensitivity to smooth signals ($\lambda \approx 1$).
A large value reflects certain strong low-frequency influence, indicating potential over-smoothness. Therefore, we have
\begin{align}
    g_l^{\mathrm{lp}}(\lambda) \approx (1+2\xi)\rho(\mL)^{d-1}.
\end{align}
When the graph has strong low-frequency components, $g_l^{\mathrm{lp}}(\lambda)$ will be high. A layer with high ${g_l^{\mathrm{lp}}}/ { g_l^{\mathrm{id}}}$ usually indicates high susceptibility to over-smoothing.
If ${g_l^{\mathrm{lp}}}/{g_l^{\mathrm{id}}}$ grows with $l$, it means the smoothness-induced sensitivity compounds with depth—a clear risk signal for deep GNNs.
Such a design of $\psi(\lambda)$ satisfies $\Abs{\psi(\lambda)}\ge \Abs{\lambda}$ while amplifying low $\lambda$, so that the sensitivity bound is specific for over-smoothness.

There are also some alternative low-pass filter designs with tunable emphasis on smoothness, e.g.,  $\psi(\lambda)=\Abs{\lambda}+\xi(1-\lambda^2)$ with $\xi \ge 0$. Thus, $g_l^{\mathrm{lp}}(\lambda)$ depends on $\xi$ and the specific eigenvalue distribution, which can be tuned to diagnose if mid-range frequencies contribute significantly to sensitivity.

It is worth noting that due to $\Abs{\psi(\lambda)}\ge \Abs{\lambda}$, the low-pass filter always excesses the identity one, where the ``excess'' term in the bound is not just a penalty—it quantifies a design trade-off to answer the question as to how much worse could the generalization be if the smoothness is enforced for better empirical performance. 

\textbf{High-pass filter.} When $\psi(\lambda)=\frac{1+2\xi}{1+\xi(1+\lambda)}$ with $\xi>0$, it probes sensitivity to smooth signals ($\lambda \approx -1$).
A large result signals strong high-frequency influence, indicating sensitivity to heterophily or local noise. Therefore, we have
\begin{align}
    g_l^{\mathrm{hp}}(\lambda)\approx (1+2\xi)\rho(\mL)^{d-1}.
\end{align}
Such a design of $\psi(\lambda)$ satisfies $\Abs{\psi(\lambda)}\ge \Abs{\lambda}$ while amplifying high $\lambda$, so that we end up with a sensitivity bound specifically for heterophily.

Some other high-pass filter designs could aim to amplify negative eigenvalues, e.g., $\psi(\lambda)=\Abs{\lambda}+\xi(1+\lambda)$ with $\xi \ge 0$. Such a filter selectively amplifies the influence of negative eigenvalues. Therefore, a large bound indicates that operations in strongly heterophilic settings are highly sensitive.

\textbf{Impact on generalization.} 
For the above low-/high-pass filter designs, although the cost is a slightly larger bound, the payoff is a clear, spectral explanation for why the bound is large and what that implies about the GNN's behavior on a given graph via the spectral sensitivity function $g_l(\lambda)$.

If $g_l^{\mathrm{lp}} \gg g_l^{\mathrm{id}}$ for later layers, it follows that GNN's sensitivity is dominated by low frequencies, thus it indicates that low-frequency (smooth) components dominate the sensitivity at depth. This is a spectral signature of over-smoothing.
If $g_l^{\mathrm{hp}} \gg g_l^{\mathrm{id}}$ for early layers, sensitivity is dominated by high frequencies, which indicates high-frequency (oscillatory) components drive sensitivity. This suggests the model is highly sensitive to local variations, a potential indicator of heterophily or over-squashing.

Regarding generalization bounds, deep GNNs on homophilic graphs generalize well if later layers have moderate weights, and the generalization on heterophilic graphs depends critically on layer parity and weight distribution.
In practical designs, it is suggested choosing filters that minimize $g_l(\lambda)$ for the graph's spectral profile. For homophily, low-pass filters yield tighter bounds; for heterophily, high-pass or adaptive filters are better. If $g_l(\lambda)$ grows with $l$, shallow networks may generalize better; If it decays, depth is beneficial.
The tightest bound $\psi(\lambda)=\lambda$ tells the worst-case and the global spectral limit, whilst a ``suboptimal'' bound can tell local connectivity, spectral energy distribution, or the cost of smoothing.

Choosing a ``suboptimal'' sensitivity matrix $\mA_l$ that leads to a looser bound can be a powerful analytical tool. The goal shifts from minimizing a scalar to structuring the bound to reveal how specific graph properties influence generalization. A strategically ``worse'' design for $\mA_l$ can make the bound explicitly depend on more interpretable or controllable graph properties, offering these key insights.

\section{Conclusion}
This work conducted PAC-Bayesian generalization analysis for graph convolutional networks in graph classification, explicitly accounting for graph structures in generalization bounds. By extending a unified PAC-Bayesian framework for deep neural networks to graph-structured models, we provided a principled approach that accommodates anisotropic Gaussian posteriors, structured sensitivity to weight perturbations, and explicit incorporation of graph topology. The resulting topology-aware generalization bounds capture either spatial or spectral characteristics of GCNs and recover several existing PAC-Bayesian bounds as special cases, while offering guarantees that are tighter than prior results. Beyond their theoretical sharpness, the proposed bounds offer improved interpretability by clarifying how graph structure and heterogeneous parameter sensitivity jointly influence generalization in graph neural networks.
Future works consist of the extension of the proposed framework to other classes of GNNs, such as message-passing neural networks, attention-based GNNs, and architectures with adaptive or learnable graph operators, as well as the integration of adversarial settings, including adversarial perturbations of node/edge features and/or graph structure. 

\section{Appendix}
\subsection{Useful Lemmas}
\begin{lemma} \label{lemma:jacobian-matrix}
For a graph convolutional network (GCN) with the node embeddings $\mH_{l}= \phi({\mL} \mH_{l-1} \mW_l)$ at the $l$-th graph convolutional layer ($l < d$) and the last readout layer $f_{\vw}^T(\mX) = \frac{1}{n} \vone^T \mH_{d-1} \mW_d$, the Jacobian of the output with respect to weights at the $l$-th layer, i.e., $\mJ_l = \frac{\partial f_{\vw}}{\partial \mathrm{vec}(\mW_l)}$, can be given by
\begin{align}
    \mJ_l &= \frac{1}{n}\big(\mW_d^T \otimes \vone^T\big) \left(\prod_{k=d-1}^{l+1} \mB_k(\mW_k^T \otimes {\mL})\right)\nn\\
    & \qquad \qquad \cdot\mB_l \Big(\mI_{h_l} \otimes  ({\mL} \mH_{l-1})\Big). 
\end{align}
\end{lemma}
\begin{proof}
    For the readout layer, we have
\begin{align}
    \mathrm{vec}(f_{\vw}^T(\mX))&= \Big(\frac{1}{n}\mI_{h_d} \otimes (\vone^T \mH_{d-1} )\Big) \mathrm{vec}(\mW_d)
\end{align}
which yields
\begin{align}
    \mJ_d =  \frac{1}{n} \mI_{h_d} \otimes (\vone^T \mH_{d-1} ) \in \R^{h_d \times h_{d-1}h_d}.
\end{align}
For the graph convolutional layer-$l$ with $l<d$,
define $\mB_l=\mathrm{diag}(\mathrm{vec}(\phi'({\mL} \mH_{l-1}\mW_l)) \in \R^{nh_l \times nh_l}$, where $\phi'$ is the derivative of the activation function. When $\phi$ is ReLU, then the diagonal elements in $\mB_l$ is either 0 or 1. For the last readout layer, the output difference is $\Delta f_{\vw}^T(\mX) = \frac{1}{n} \vone^T \Delta \mH_{d-1} \mW_d$, where $\Delta \mH_{k}$ is the output difference at the $k$-th layer due to the accumulated weight perturbation from the previous layers, resulting in
\begin{align}
    \mathrm{vec}(\Delta f_{\vw}(\mX))&= \Big(\frac{1}{n}\mW_{d}^T \otimes \vone^T \Big) \mathrm{vec}(\Delta \mH_{d-1}).
\end{align}
Let $\Delta \mW_l$ be the weight perturbation of $\mW_l$.
When $k=d-1,\dots,l+1$, define
\begin{align}
 \Delta \mZ_k = {\mL} \Delta\mH_{k-1}\mW_k + {\mL} \mH_{k-1}\Delta \mW_k  
\end{align}
where the first term is due to the propagated perturbation from the previous layer, and the second term is the direct weight perturbation at the current layer.
So, by vectorization and applying the ReLU activation, we have
\begin{align}
   \vec(\Delta\mH_{k})&= \mB_k \vec(\Delta \mZ_k) \\
   &= \mB_k(\mW_k^T \otimes {\mL}) \vec(\Delta\mH_{k-1}) \nn \\
   &\qquad \qquad + \mB_k(\mI_{h_k} \otimes {\mL}\mH_{k-1}) \vec(\Delta \mW_k).
\end{align}
Therefore, we have
\begin{align} 
    \frac{\partial  \vec(\mH_{k})}{\partial  \vec(\mH_{k-1})}=\frac{  \vec(\Delta\mH_{k})}{  \vec(\Delta\mH_{k-1})}\Big|_{\Delta\mH_{k-1} \to \mathbf{0}} = \mB_k(\mW_k^T \otimes {\mL})\\
    \frac{\partial  \vec(\mH_{l})}{\partial  \vec(\mW_{l})}=\frac{  \vec(\Delta\mH_{l})}{  \vec(\Delta \mW_{l})}\Big|_{\Delta  \mW_{l} \to \mathbf{0}} = \mB_l(\mI_{h_l} \otimes {\mL}\mH_{l-1}). \
\end{align}

Therefore,
\begin{align}
    \mJ_l &= \frac{\partial\vec( f_{\vw})}{\partial\vec( \mH_{d-1})} \prod_{k=d-1}^{l+1}\frac{\partial\vec( \mH_{k})}{\partial\vec( \mH_{k-1})} \frac{\partial\vec( \mH_{l})}{\partial\vec(\mW_l)} \\
    &= \frac{1}{n}\big(\mW_d^T \otimes \vone^T\big) \left(\prod_{k=d-1}^{l+1} \mB_k(\mW_k^T \otimes {\mL})\right)\nn\\
    & \qquad \qquad \cdot\mB_l \Big(\mI_{h_l} \otimes  ({\mL} \mH_{l-1})\Big)
\end{align}
with $\mH_{l-1}=\phi({\mL} \mH_{l-2} \mW_{l-1})$.
\end{proof}

\begin{lemma} \label{lemma:fun-properties}
    Given a matrix $\mX \in \R^{m \times n}$ and $\alpha>0$ is independent of $\mX$ such that $\mR=(\mI+\alpha \mX^T \mX)^{-1}$, we have 
    \begin{align}
        f(\mX) = \Tr(\mR) - \log \det \mR - n \le \alpha \Tr(\mX^T \mX)
    \end{align}
    where the equality holds if and only if $\mX=\mathbf{0}$.
\end{lemma}
\begin{proof}
    Let $\lambda_i \ge 0$ be the eigenvalues of $\mX^T \mX$ and define $\mu_i=\frac{1}{1+\alpha \lambda_i} \in (0,1]$. Then
    \begin{align}
        f(\mX)&=\sum_{i=1}^n (\mu_i - \log \mu_i - 1)\\
        &= \sum_{i=1}^n \left( \frac{1}{1+\alpha \lambda_i} + \log (1+\alpha \lambda_i) - 1 \right)\\
        &= \sum_{i=1}^n h(\alpha \lambda_i)
    \end{align}
    where $h(u)=\frac{1}{1+u}+\log(1+u)-1$ for $u \ge 0$. It is readily verified that $h(u) \le u$ for all $u\ge0$, where the equality holds if and only if $u=0$. Thus, we have
    \begin{align}
        f(\mX)= \sum_{i=1}^n h(\alpha \lambda_i) \le \sum_{i=1}^n \alpha \lambda_i = \alpha \Tr(\mX^T \mX)
    \end{align}
    where the equality holds if and only if $\lambda_i=0$ for all $i$, i.e., $\mX=\mathbf{0}$. This completes the proof.
\end{proof}

\begin{lemma} \label{lemma:kron-properties}
    For positive semidefinite matrices $\mA$, $\mB$, and $\mC$, the following inequality holds
    \begin{align}
        \mA \otimes \mB \preceq \mA \otimes \mC \quad \text{if and only if} \quad \mB \preceq \mC
    \end{align}
    where $\otimes$ is the Kronecker product and $\preceq$ denotes the Lowner order, i.e., $\mX \preceq \mY$ means $\mY-\mX$ is positive semidefinite.
\end{lemma}
\begin{proof}
    \textbf{Forward pass.} Given $\mA \otimes \mB \preceq \mA \otimes \mC$, for any vectors $\vx$ and $\vy$, we have
    \begin{align}
        (\vx \otimes \vy)^T (\mA \otimes \mB) (\vx \otimes \vy) \le (\vx \otimes \vy)^T (\mA \otimes \mC) (\vx \otimes \vy).
    \end{align}
    Using the identity $(\mU \otimes \mV) (\vx \otimes \vy)=(\mU \vx) \otimes (\mV \vy)$, we have
    \begin{align}
        (\vx^T \mA \vx)(\vy^T \mB \vy) \le (\vx^T \mA \vx)(\vy^T \mC \vy).
    \end{align}
    As there exist some $\vx$ such that $\vx^T \mA \vx > 0$, then we have, for any vector $\vy$, that
    \begin{align}
        \vy^T \mB \vy \le \vy^T \mC \vy,
    \end{align}
    which yields $\mB \preceq \mC$.

    \textbf{Backward pass.} Given $\mB \preceq \mC$, it follows $\mC-\mB \succeq \mathbf{0}$ and $\mA \otimes (\mC-\mB)$ is also positive semidefinite, i.e.,
    \begin{align}
        \mA \otimes (\mC-\mB) = \mA \otimes \mC- \mA \otimes \mB \succeq \mathbf{0},
    \end{align}
    which implies $\mA \otimes \mB \preceq \mA \otimes \mC$.
\end{proof}

\begin{lemma}
\label{lemma:kron-norm-eq}
    For two matrices $\mA$ and $\mB$, we have $\Norm{\mA \otimes \mB}_2 = \Norm{\mA}_2 \Norm{\mB}_2$ and $\Norm{\mA \otimes \mB}_F = \Norm{\mA}_F \Norm{\mB}_F$.
\end{lemma}
\begin{proof}
    Let the SVD of $\mA$ and $\mB$ be $\mA=\sum_{i} \sigma_i \vu_i \vv_i^T$ and $\mB=\sum_i \lambda_i \vx_i \vy_i^T$, where $\{\vu_i\}$, $\{\vv_i\}$, $\{\vx_i\}$, and $\{\vy_i\}$ are their left and right singular vectors. Thus, we have
    \begin{align}
        \mA \otimes \mB &= (\sum_i \sigma_i \vu_i \vv_i^T) \otimes (\sum_i \vx_i \vy_i^T)\\
        &= \sum_i \sum_j  \sigma_i \lambda_j (\vu_i \otimes \vx_j)(\vv_i^T \otimes \vy_j^T)\\
        &= \sum_i \sum_j \sigma_i \lambda_j (\vu_i \otimes \vx_j)(\vv_i \otimes \vy_j)^T.
    \end{align}
    As any two vectors in $\{\vv_i \otimes \vx_j\}_{i,j}$ are orthogonal for different $(i,j)$, so are $\{\vu_i \otimes \vy_j\}_{i,j}$, it follows that the two sets of orthogonal vectors are eigenvectors of $\mA \otimes \mB$ with corresponding singular values $\{\sigma_i \lambda_j\}_{i,j}$. Therefore, we have
    \begin{align}
        \Norm{\mA \otimes \mB}_2 &= \max_i \sigma_i \max_j \lambda_j = \Norm{\mA}_2 \Norm{\mB}_2\\
        \Norm{\mA \otimes \mB}_F^2 &= \sum_i \sum_j \sigma_i^2  \lambda_j^2 = \Norm{\mA}_F^2 \Norm{\mB}_F^2.
    \end{align}
\end{proof}

\begin{lemma} \label{lemma:kron-norm}
    Given matrices $\mA$, $\mX$, and $\mB$, we have $\Norm{\mA \mX \mB}_F \le \Norm{\mA}_2 \Norm{\mB}_2 \Norm{\mX}_F$.
\end{lemma}
\begin{proof}
For any matrix $\mX$, the vectorization operator $\vec(\cdot)$ stacks the columns of a matrix into a column vector so that $\Norm{\mX}_F^2=\Norm{\vec(\mX)}_2^2$. Then, it follows that
\begin{align}
    \Norm{\mA \mX \mB}_F^2 &= \Norm{\vec(\mA \mX \mB)}_2^2\\
    &=\Norm{(\mB^T \otimes \mA)\vec(\mX)}_2^2 \\
    &\le  \Norm{\mB^T \otimes \mA}_2^2 \Norm{\vec(\mX)}_2^2\\
    &=\Norm{\mA}_2^2 \Norm{\mB}_2^2 \Norm{\vec(\mX)}_2^2\\
    &=\Norm{\mA}_2^2 \Norm{\mB}_2^2 \Norm{\mX}_F^2
\end{align}
where the inequality is due to the fact that the Euclidean norm of a matrix-vector product is bounded by the spectral norm of the matrix times the Euclidean norm of the vector.
\end{proof}

\begin{lemma} \label{lemma:graph-properties}
    Consider a graph $G=(V,E)$ with adjacency matrix $\mA_G \in \{0,1\}^{n \times n}$ and its. Let $\mI + \mA_G$ be the adjacency matrix with self-loops, and $\Tilde{\mD}=\mI+\mD$ with $\mD=\mathrm{diag}(\mA_G\vone)$ the corresponding degree matrix.
    
    For vanilla GCNs ${\mL} = \Tilde{\mD}^{-\frac{1}{2}} (\mI + \mA_G) \Tilde{\mD}^{-\frac{1}{2}}$, we have
    \begin{align}
        \rho(\mL)=\Norm{{\mL}}_2 = 1,\\
        \frac{\sum_i \sqrt{1+\mD_{ii}}}{\sqrt{\sum_i (1+\mD_{ii})}} \le \Norm{{\mL}^{d-1}\vone}_2 &\le \sqrt{n}.
    \end{align}
    
    For random walk GCNs $\mL_{\mathrm{rw}}=\Tilde{\mD}^{-1} (\mI + \mA_G)$, we have
    \begin{align}
        \rho(\mL_{\mathrm{rw}})= 1, &\quad \Norm{\mL_{\mathrm{rw}}^{d-1}\vone}_2^2 = n \\ 1 \le \Norm{{\mL_{\mathrm{rw}}}}_2 &\le \sqrt{\frac{D_{\max}+1}{D_{\min}+1}}.
    \end{align}
\end{lemma}
\begin{proof}
We first show that $\mL$ and $\mL_{\mathrm{rw}}$ have the same set of eigenvalues, followed by the separate proofs.

Note that $\mL_{\mathrm{rw}}=\Tilde{\mD}^{-1}(\mI+\mA_G)$ is a random walk matrix, which is a transition matrix of a lazy random walk, such that at each step one can stay still or move to a neighbor. The elements in $\mL_{\mathrm{rw}}$ are the transition probabilities. It is readily verified that $\mL$ and $\mL_{\mathrm{rw}}$ are similar in the sense that
\begin{align}
    \mL_{\mathrm{rw}} &= \Tilde{\mD}^{-1}(\mI+\mA_G) \\
    &= \Tilde{\mD}^{-\frac{1}{2}} \Tilde{\mD}^{-\frac{1}{2}}(\mI+\mA_G) \Tilde{\mD}^{-\frac{1}{2}} \Tilde{\mD}^{\frac{1}{2}} = \Tilde{\mD}^{-\frac{1}{2}} \mL \Tilde{\mD}^{\frac{1}{2}}
\end{align}
so that they share the same set of eigenvalues $\{\lambda_i\}_{i=1}^n$.
As $\mL_{\mathrm{rw}}$ is diagonally similar to a symmetric matrix $\mL$, the eigenvalues are real-valued.
Since $\mL_{\mathrm{rw}}$ is a row-stationary matrix, by the Perron-Frobenius theorem, we have $0 \le \lambda_i \le 1$ for all $i$.

\paragraph{Vanilla GCNs}
Given that the eigenvalues of ${\mL}$ lie in the range of [0,1], we have    
    \begin{align}
        \Tilde{\mD}^{-\frac{1}{2}} (\mI + \mA_G) \Tilde{\mD}^{-\frac{1}{2}} \Tilde{\mD}^{\frac{1}{2}} \vone &= \Tilde{\mD}^{-\frac{1}{2}} (\mI + \mA_G)  \vone \\
        &= \Tilde{\mD}^{-\frac{1}{2}} \Tilde{\mD}  \vone = \Tilde{\mD}^{\frac{1}{2}} \vone,
    \end{align}
    which indicates $\Tilde{\mD}^{\frac{1}{2}} \vone$ is an eigenvector of ${\mL}$ with eigenvalue 1. Since $\mL$ is symmetric, we have $\rho(\mL)=\Norm{{\mL}}_2 = \lambda_{\max}({\mL})=1$.

    Let $(\lambda_i,\vv_i)$ be eigenpairs of ${\mL}$ with the eigenvalues $\lambda_i \in [0,1]$. If letting $\lambda_1=1$ be the largest eigenvalue, then $\vv_1=\frac{\Tilde{\mD}^{\frac{1}{2}} \vone}{\Norm{\Tilde{\mD}^{\frac{1}{2}} \vone}_2}$ is the normalized eigenvector. Let $\vone=\sum_{i=1}^n \alpha_i \vv_i$ with $\{\alpha_i\}_{i=1}^n$ being the coefficients of projecting $\vone$ on the eigenspace $\mathrm{span}(\{\vv_i\}_{i=1}^n)$. Therefore,
    \begin{align}
        {\mL}^{d-1}\vone &= \left(\sum_{i=1}^n \lambda_i^{d-1} \vv_i \vv_i^T \right) \left(\sum_{j=1}^n {\alpha_j} \vv_j\right) \\
        &= \sum_{i=1}^n {\alpha_i} \lambda_i^{d-1} \vv_i.
    \end{align}
    Since $\lambda_i \in [0,1]$, we have
    \begin{align}
        \Norm{{\mL}^{d-1}\vone}_2^2 = \sum_{i=1}^n {\alpha_i^2} \lambda_i^{2(d-1)} \le \sum_{i=1}^n {\alpha_i^2} = \Norm{\vone}_2^2 = n. 
    \end{align}
As a side remark, the over-smoothing occurs as $d$ increases because all terms where $\Abs{\lambda_i}<1$ decay to zero, leaving only the components aligned with $\Abs{\lambda_i}=1$.
    
    On the other hand
    \begin{align}
        \Norm{{\mL}^{d-1}\vone}_2 \ge \alpha_1 \lambda_1^{d-1} =\alpha_1
    \end{align}
    where 
    \begin{align}
        \alpha_1 = \vone^T \vv_1 = \frac{\vone^T\Tilde{\mD}^{\frac{1}{2}} \vone}{\Norm{\Tilde{\mD}^{\frac{1}{2}} \vone}_2} = \frac{\sum_i \sqrt{1+\mD_{ii}}}{\sqrt{\sum_i (1+\mD_{ii})}},
    \end{align}
    given $\vone=\sum_{i=1}^n \alpha_i \vv_i$. Note here that the upper and lower bounds of $\Norm{{\mL}^{d-1}\vone}_2$ coincide when $G$ is a regular graph with the same degree for all nodes. 

\paragraph{Random walk GCNs}
Note that ${\mL}_{\mathrm{rw}}$ shares the same eigenvalues $\lambda_i \in [0,1]$, yet the different eigenvectors due to the asymmetry. As ${\mL}_{\mathrm{rw}}$ is a row-stationary matrix, we have $\mL \vone = \vone$, i.e., $\lambda_1=1$. Thus, we have $\rho({\mL}_{\mathrm{rw}})=1$ and
\begin{align}
    \mL^{d-1} \vone = \vone, \text{ for any $d$}.
\end{align}
This yields $\Norm{\mL^{d-1} \vone}_2=\Norm{\vone}_2=\sqrt{n}$.
Further, we have the lower bound $\Norm{{\mL}_{\mathrm{rw}}}_2 \ge \rho({\mL}_{\mathrm{rw}})=1$ and the upper bound
\begin{align}
    \Norm{{\mL}_{\mathrm{rw}}}_2 &= \Norm{\Tilde{\mD}^{-\frac{1}{2}} \mL \Tilde{\mD}^{\frac{1}{2}}} \le \Norm{\Tilde{\mD}^{-\frac{1}{2}}}_2 \Norm{\mL}_2 \Norm{ \Tilde{\mD}^{\frac{1}{2}}}_2 \\
    &=\sqrt{\frac{D_{\max}+1}{D_{\min}+1}}.
\end{align}

This completes the proof.
\end{proof}

\begin{lemma} \label{lemma:GCN-pert-bound-new}
Let $f_{\vw}(\mX)$ be a GCN model with $d$ layers, such that $\Norm{\mX}_{2,\infty}\le B$. Let $\mL$ be the graph propagation matrix. Given any weight perturbation $\mU_l$ to the weight $\mW_l$, we have the output difference bounded by
    \begin{align} \nn
        \Norm{f_{\vw+\vu}(\mX)-f_{\vw}(\mX)}_2^2 \le d B^2 \Norm{{\mL}}_2^{2d-2} \prod_{l=1}^d \Norm{\mW_l}_2^2 \sum_{l=1}^d \frac{\Norm{\mU_l}_2^2}{\Norm{\mW_l}_2^2}.
    \end{align} 
\end{lemma}
\begin{proof}
Let 
\begin{align}
    {\mA}_l= \alpha_l \big(\mI_{h_l} \otimes ({\mL}\mH_{l-1})\big)
\end{align}
with 
\begin{align}
    \alpha_l &= \frac{\sqrt{d}}{n} \Norm{\vone^T\mL^{d-l-1}}_2 \prod_{k=d}^{l+1} \Norm{\mW_k}_2.
\end{align}
    Therefore, we have
\begin{align}
    \Norm{{\mA}_l \vu_l}_2^2 &= \Norm{{\mA}_l \vec(\mU_l)}_2^2\\
    &=\alpha_l^2 \Norm{\vec({\mL}\mH_{l-1}\mU_l)}_2^2\\
    &=\alpha_l^2 \Norm{{\mL}\mH_{l-1}\mU_l}_F^2\\
    &\le \alpha_l^2 \Norm{{\mL}\mH_{l-1}}_F^2 \Norm{\mU_l}_2^2\\
    &\le \alpha_l^2 \Norm{{\mL}}_2^2 \Norm{\mH_{l-1}}_F^2 \Norm{\mU_l}_2^2\\
    &\le \alpha_l^2 \Norm{\mX}_F^2 \Norm{{\mL}}_2^{2l} \prod_{k=1}^{l-1} \Norm{\mW_k}_2^2 \Norm{\mU_l}_2^2\\
    &\le \frac{d}{n}\Norm{\mX}_F^2 \Norm{{\mL}}_2^{2d-2} \prod_{k\neq l}^d \Norm{\mW_k}_2^2 \Norm{\mU_l}_2^2\\
    &\le d B^2 \Norm{{\mL}}_2^{2d-2} \prod_{k\neq l}^d \Norm{\mW_k}_2^2 \Norm{\mU_l}_2^2
\end{align}
where $\Norm{\mH_{l}}_F \le \Norm{\mL}_2 \Norm{\mH_{l-1}}_F \Norm{\mW_l}_2 $ due to Lemma \ref{lemma:kron-norm}, and $\alpha_l \le \sqrt{\frac{d}{n}} \Norm{\mL}_2^{d-l-1} \prod_{k=d}^{l+1} \Norm{\mW_k}_2$.

Therefore, we have
\begin{align}
    \MoveEqLeft \Norm{f_{\vw+\vu}(\mX)-f_{\vw}(\mX)}_2^2 \\
    &= \Norm{\sum_{l=1}^d \mJ_l \vu_l}_2^2 + o(\Norm{\vu_l}_2^2)\\
    &\le d \sum_{l=1}^d \Norm{ \mJ_l \vu_l}_2^2 \le \sum_{l=1}^d \Norm{ \mA_l \vu_l}_2^2\\
    &\le d B^2 \Norm{{\mL}}_2^{2d-2} \prod_{l=1}^d \Norm{\mW_l}_2^2 \sum_{l=1}^d \frac{\Norm{\mU_l}_2^2}{\Norm{\mW_k}_2^2}.
\end{align}
Compared with the GCN perturbation bounds in \cite{liao2020pac,sun2024pac}, we have the only difference of $d\sum_{l=1}^d \frac{\Norm{\mU_l}_2^2}{\Norm{\mW_k}_2^2}$ versus $(\sum_{l=1}^d \frac{\Norm{\mU_l}_2}{\Norm{\mW_k}_2})^2$, where the former is slightly larger than the latter. Especially, when $\frac{\Norm{\mU_l}_2}{\Norm{\mW_k}_2}$ is equal for all $l$, our bound is similar to those in \cite{liao2020pac,sun2024pac} without the $e^2$ factor. Nevertheless, the above GCN perturbation bound removes the factor of $e^2$ and the required condition of $\Norm{\mU_l} \le \frac{1}{d} \Norm{\mW_l}$ for all $l$. For the GCN models, $d$ is usually not big, so the above bound is comparable to those in \cite{liao2020pac,sun2024pac}.
\end{proof}

\bibliographystyle{ieeetr}
\bibliography{PAC-Bayes}

\end{document}